\documentclass[twoside]{article}
\usepackage[accepted]{aistats2023}
\usepackage{microtype}
\usepackage{graphicx}
\usepackage{subfigure}
\usepackage{booktabs} % for professional tables
\usepackage{shortbold}
\usepackage{hyperref}

\usepackage{amsmath}
\usepackage{multicol}
\usepackage{amssymb}
\usepackage{mathtools}
\usepackage{amsthm}
\usepackage{shortbold}
\usepackage[capitalize,noabbrev]{cleveref}
\usepackage[round]{natbib}

\theoremstyle{plain}
\newtheorem{theorem}{Theorem}[section]

\theoremstyle{definition}
\newtheorem{definition}[theorem]{Definition}

\theoremstyle{remark}
\newtheorem{remark}[theorem]{Remark}
\usepackage{xr}
\usepackage[textsize=tiny]{todonotes}

\begin{document}

\twocolumn[

\aistatstitle{Adversarial robustness of VAEs through the lens of local geometry}

\aistatsauthor{Asif Khan \And Amos Storkey}

\aistatsaddress{School of Informatics, \\ University of Edinburgh, U.K. \And School of Informatics,\\ University of Edinburgh, U.K. }]

\begin{abstract}
In an unsupervised attack on variational autoencoders (VAEs), an adversary finds a small perturbation in an input sample that significantly changes its latent space encoding, thereby compromising the reconstruction for a fixed decoder. A known reason for such vulnerability is the distortions in the latent space resulting from a mismatch between approximated latent posterior and a prior distribution. Consequently, a slight change in an input sample can move its encoding to a low/zero density region in the latent space resulting in an unconstrained generation. This paper demonstrates that an optimal way for an adversary to attack VAEs is to exploit a directional bias of a stochastic pullback metric tensor induced by the encoder and decoder networks. The pullback metric tensor of an encoder measures the change in infinitesimal latent volume from an input to a latent space. Thus, it can be viewed as a lens to analyse the effect of input perturbations leading to latent space distortions. We propose robustness evaluation scores using the eigenspectrum of a pullback metric tensor. Moreover, we empirically show that the scores correlate with the robustness parameter $\beta$ of the $\beta-$VAE. Since increasing $\beta$ also degrades reconstruction quality, we demonstrate a simple alternative using \textit{mixup} training to fill the empty regions in the latent space, thus improving robustness with improved reconstruction.
% Variational autoencoders (VAEs) are susceptible to adversarial attacks. An adversary can find a small perturbation in an input sample that can result in big changes in its latent space encoding, thereby compromising the reconstruction for a fixed decoder. A known reason for such vulnerability is the distortions in the latent space resulting from a mismatch between approximated latent posterior and a prior distribution. Consequently, a slight change in input can move its encoding to a low/zero density region in the latent space. This paper demonstrates that an optimal way for an adversary to find a perturbation is to exploit a directional bias of a stochastic pullback metric tensor induced by the encoder network. The pullback metric tensor measures the change in infinitesimal latent volume from an input to a latent space. Thus, it can be viewed as a lens to analyse the effect of input perturbations leading to latent space distortions. We propose robustness evaluation scores using the eigenspectrum of a pullback metric. Moreover, we empirically show that the scores correlate with the robustness parameter $\beta$ of the $\beta-$VAE. Since increasing $\beta$ also results in mode collapse, we demonstrate a simple strategy to fill the empty regions in the latent space, thus bridging the gap between latent posterior and prior.
\end{abstract}

\section{INTRODUCTION}
\label{sec:intro}
Variational autoencoders (VAEs) belong to a class of deep generative models with a stochastic encoder-decoder network~\citep{kingma2013auto}. The encoder parameterises the variational distribution over latent variables conditioned on data samples, and the decoder estimates the data distribution through the latent distribution. Thus, VAEs serve a dual purpose of estimating data density and providing a rich representation space with uncertainty quantification. Recently several works have shown the application of VAEs to high-fidelity image generation~\citep{vahdat2020nvae}, music generation~\citep{roberts2017hierarchical}, video generation~\citep{wu2021greedy}, and many others. However,
like other machine learning models, VAEs are also susceptible to adversarial attacks, as demonstrated in several recent works~\citep{tabacof2016adversarial, gondim2018adversarial,willetts2019improving,kos2018adversarial,kuzina2021diagnosing}. In a typical setup, an adversary can attack a VAE by learning a small perturbation to an input that will lead to a large change in its latent encoding. This mechanism takes the form of an optimisation problem (introduced later in Equation~\ref{eq:optim}), which is generally solved using stochastic gradient methods~\citep{chakraborty2018adversarial,szegedy2013intriguing}. For a more comprehensive overview of attacks on VAEs and other generative models, we refer readers to~\citep{sun2021adversarial}.

The primary reason for the vulnerability of VAEs to attacks is the distortion in the latent space resulting from the mismatch between approximated posterior and latent space prior, also known as the \textit{posterior-prior gap}. Hence, the latent space is non-smooth, and the representations of similar inputs tend to be significantly distant in the space under the Euclidean metric. Methods such as~\citep{mathieu2019disentangling,willetts2019improving} have emphasised the importance of reducing the \textit{posterior-prior} for learning disentangled latent space that also improves the robustness of VAEs. $\beta-$ VAE~\citep{higgins2016beta} formulation introduces a parameter $\beta$ to directly control the gap. Other methods~\citep{chen2018isolating, kim2018disentangling, esmaeili2019structured} utilise the total correlation (TC) term to disentangle the latent coordinates of VAEs that also smooths the latent space proving helpful in improving the robustness. The limitation of these approaches is that they cause over-smoothing of the reconstructed samples and require a careful training mechanism to balance the regularisation term.~\citet{willetts2019improving} address this problem by introducing a TC term in hierarchical VAEs that provide robustness along with sharp reconstruction. However, most robustness methods use regularisation terms, which do not provide meaningful insights for quantifying robustness. Therefore, the notion of a small change in the input to the large change in latent space is not well established. Moreover, comparing these schemes relies on the visual inspection of distorted images at different magnitudes of adversarial loss. 

Much recently~\citet{camuto2021towards} proposed a theoretical framework that considers the uncertainty of encoder for studying the robustness of VAEs. However, the attacks in the input space do not consider the effect of geometry induced by an encoder or decoder maps. In another recent work~\citet{kuzina2021diagnosing} proposed an asymmetric KL term to capture the difference between the latent representation of input and its perturbation. They obtain the $\epsilon$ using the Jacobian of the mean latent code evaluated at the input perturbation. Nevertheless, their optimisation objective does not provide any geometrical insights. Likewise, they do not consider the contribution of the standard deviation term when computing Jacobian. Thus, do not account for the uncertainty in the representation space. Our paper shows that the geometry induced by the stochastic nature of encoder mapping provides the intuition behind the sensitivity of VAEs that can be a valuable tool for understanding robustness.

The central theme of our paper is to view the adversarial attack problem through the lens of manifold geometry. Unlike existing approaches treating the input space as Euclidean, we propose to utilise the stochastic pullback-metric tensor induced by the encoder map to measure the distance in the input space. We demonstrate that the distortion in the latent space results in a directional bias appearing in the form of an anisotropic metric tensor. We show that an optimal way for an adversary to design an attack is by moving along the dominant eigendirection of the induced metric tensor. Moreover, we propose scores to quantify robustness using the eigenspectrum of the pullback metric tensor. We hypothesise that methods reducing \textit{posterior-prior gap} improve robustness, directly influencing the induced metric tensors. To this end, we demonstrate that the proposed scores correlate with the $\beta$ parameter of $\beta-$VAE used to control robustness. To our knowledge, such a geometric view of the robustness of VAEs has not been previously investigated.

$\beta-$VAE is known to generate over-smooth samples for a high value of $\beta$. We introduce a simple mechanism using a mixup training~\citep{zhang2017mixup} that improves the robustness of VAE without much effect on reconstruction quality. Specifically, we introduce a regularisation loss term that forces the encoder to fill empty regions with linear interpolation in data space and ensures the decoder generates the respective interpolations, thus avoiding the issue of an unconstrained generation. We empirically demonstrate that such a training scheme improves the robustness measured by the proposed scores.
\begin{figure*}[t]
\begin{multicols}{2}
      \includegraphics[width=1.05\columnwidth,  trim={0cm 2cm 1cm 1cm}]{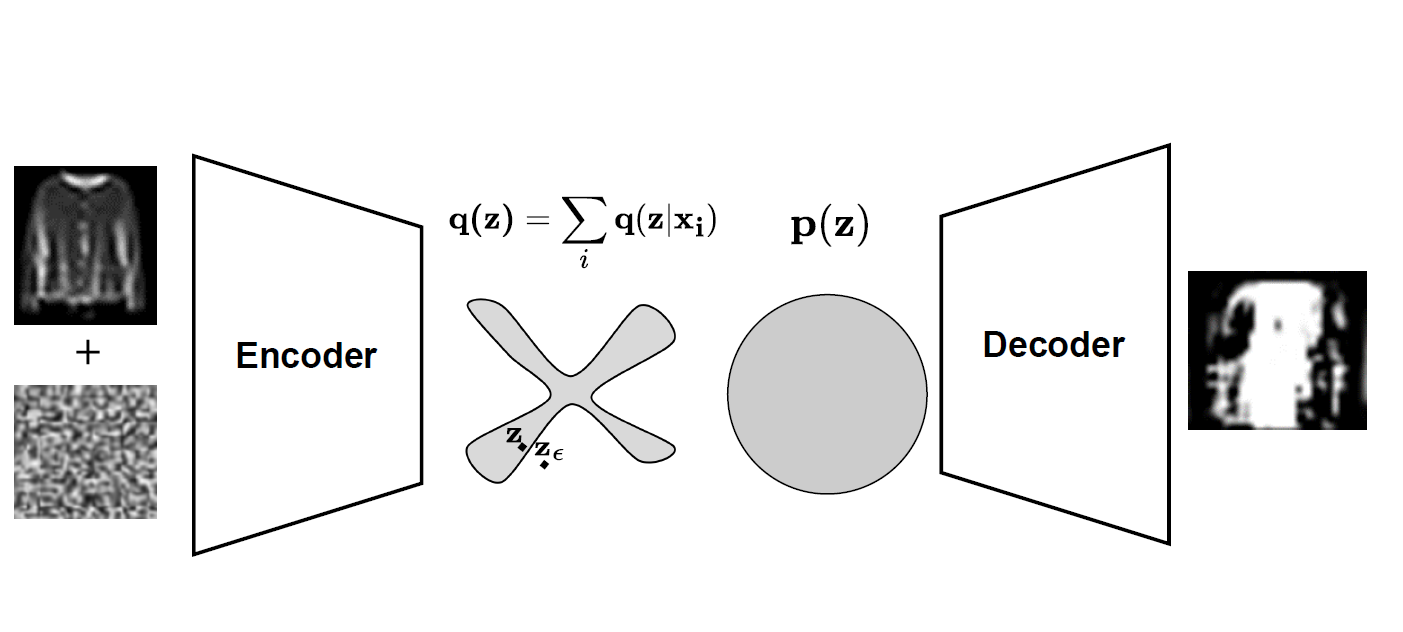}\par
    \qquad \includegraphics[width=0.7\columnwidth, trim={0cm 4cm 3cm 0}]{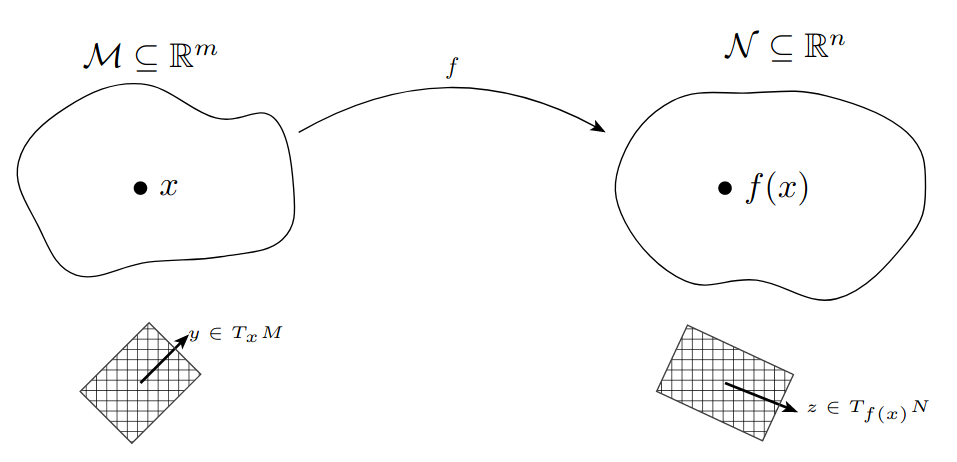} \par
\end{multicols}
\caption{\textit{Left:} Illustration that adversarial examples find non-smooth change in the latent encodings. A small perturbation in the input sample exploits a direction that maximally changes latent encoding by moving from high density to low/zero density region in the latent space. In this paper, we show an optimal perturbation can be found by moving along the eigendirection of the local pullback metric tensor of a data point. \textit{Right:} A smooth mapping $f$ from the data manifold $\mathcal{M}$ to the latent manifold $\mathcal{N}$ induces pullback metrics on $\mathcal{M}$. The Jacobian $\JBRV_{f(x)} = \frac{\partial f}{\partial \xBRV}$ is a linear map from a tangent vector $y\in T_x M$ to a tangent vector $z\in T_{f(x)} N$ that induces a pullback Riemannian metric tensor $\GBRV_{\xBRV}=\JBRV_{f(x)}^T\JBRV_{f(x)}$. The determinant of metric tensor $\GBRV_{\xBRV}$ represents the change in infinitesimal volume element when projected to the latent space.}
    \label{fig:manifoldmap}
\end{figure*}
\section{BACKGROUND}
\subsection{$\beta-$Variational autoencoder}
$\beta-$VAE~\citep{higgins2016beta} is a probabilistic encoder-decoder framework that simultaneously parameterises the latent distribution and emission distribution using deep neural networks. Consider a sample $\xBRV\in\XBRV=\mathbb{R}^N$ drawn from unknown data distribution $p(\xBRV)$, VAE learns an approximate posterior distribution $\qRV_\phi(\zBRV|\xBRV)$ over latent variables $\zBRV\in\ZBRV=\mathbb{R}^{d_z}$ using a stochastic encoder network, and an emission distribution $p_\theta(\xBRV|\zBRV)$ using a stochastic decoder network. The parameters $\theta$ of an encoder network and $\phi$  of a decoder network are learned by maximising the evidence lower bound (ELBO), 
\begin{equation}
        \mathbb{E}_{\zBRV \sim q(\zBRV|\xBRV)} [\log p_\theta(\xBRV| \zBRV)] - \beta KL [q_\phi(\zBRV|\xBRV)||p(\zBRV)]          
\label{eq:elbo_vae}    
\end{equation}
where $KL$ stands for Kullback-Leibler divergence~\citep{kullback1951information}, and a parameter $\beta$ controls the smoothness of latent distribution, setting $\beta=1$ is equivalent to a standard VAE~\citep{kingma2013auto}. 
\subsection{Adversarial attacks on VAEs}
The adversarial attacks on VAEs assume access to a pretrained encoder-decoder network. The adversary aims to exploit the capacity of VAE by finding small perturbations in an input sample that lead to a large change in its latent encoding or reconstruction. In a supervised scenario, the adversary starts with input and finds a minimal change that can match the reconstruction to the known target. In an unsupervised setting, the aim is to maximise the distance in latent codes, which will corrupt the reconstruction. Several recent developments have proposed mechanisms for designing an adversary and evaluating the robustness of existing deep generative models~\citep{tabacof2016adversarial,willetts2019improving}. Our paper takes a geometrical viewpoint of an unsupervised attack by analysing the metric tensor induced by the stochastic encoder network. We first introduce the unsupervised variational attack problem and later present our approach in Section~\ref{sec:method}. 

For an encoder neural network $f_{\theta}:\XBRV\rightarrow \ZBRV$ the unsupervised latent space attack optimises the objective~\citep{gondim2018adversarial},
\begin{align}
\label{eq:optim}
\max_{\etaB:||\etaB||_2\leq \eta_0} & \quad d (f_{\theta}(\xBRV), f_{\theta}(\xBRV+\etaB))
%\nonumber\\
%&\qquad \qquad \qquad \text{subject to} \quad ||\etaB||_2 =  \eta_{0}
\end{align}
where $\eta_0$ is a small constant that decides the severity of the attack, and $d(.,.)$ is a distance function that measures the proximity in the latent space. A common approach to finding a corruption $\etaB$ is to use stochastic gradient methods~\citep{sun2021adversarial, willetts2019improving}.

\begin{figure*}[t]
    \centering
    \small{\textbf{MNIST}: (a) Original  $\downarrow$ Reconstruction  \qquad(b) Corrupted $\delta_1$ $\downarrow$ Reconstruction \qquad(c) Corrupted $\delta_2$ $\downarrow$ Reconstruction}
   \includegraphics[trim={0 0.15cm 1.2cm 0cm},height=3.8cm, width=0.84\linewidth]{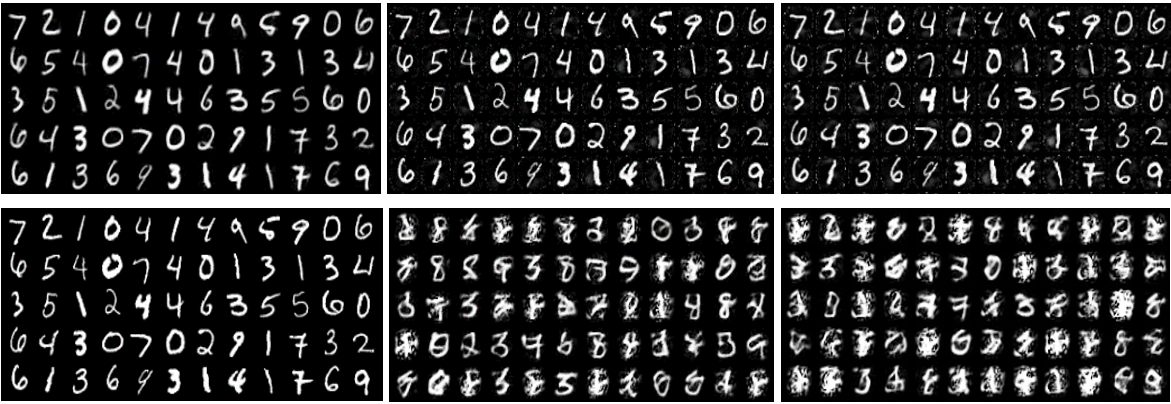}

    \small{\textbf{FMNIST}: (d) Original  $\downarrow$ Reconstruction  \qquad(e) Corrupted $\delta_1$ $\downarrow$ Reconstruction \qquad(f) Corrupted $\delta_2$ $\downarrow$ Reconstruction}
    \includegraphics[trim={0 0.15 1.2cm 0cm}, height=3.8cm, width=0.84\linewidth]{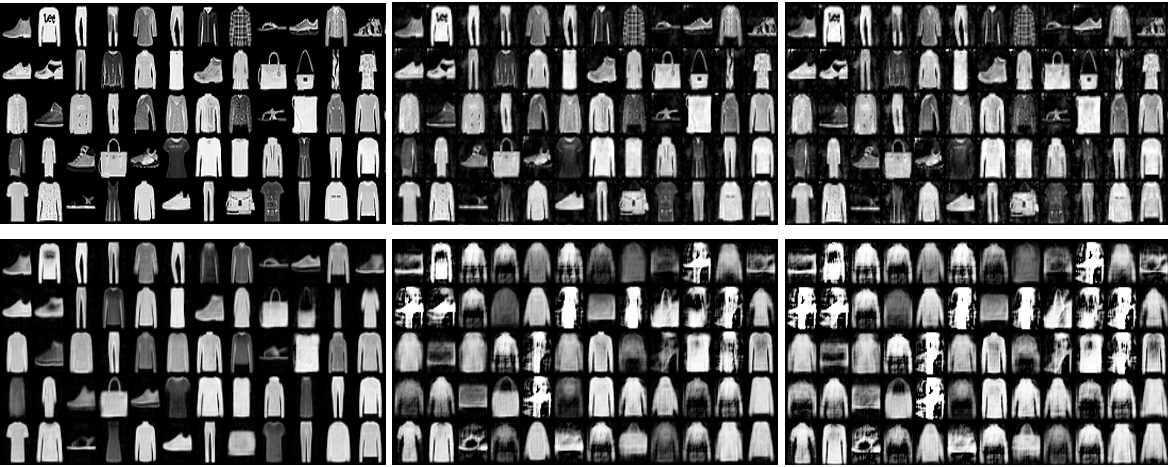}
    \caption{Illustration of adversarial attack along the dominant eigenvector of a stochastic pullback metric tensor. The first two rows are the results of MNIST data, and the bottom two are on the FashionMNIST dataset. We evaluate the reconstruction for original images and its two corrupted versions with different step sizes $\delta_1=0.5233$ and $\delta_2=0.7443$. Moving along eigendirection doesn't affect the input image but significantly changes its reconstruction.}
    \label{fig:attack}
\end{figure*}
\subsection{Latent space geometry}
In this section, we introduce definitions from Riemannian geometry relevant to the context of our work.
\begin{definition} A $n-$dimensional manifold $\gM$ is a topological space where for every $\xBRV\in\gM$ there exist a neighborhood region $\gV_{\xBRV}$ homeomorphic to $\gR^n$~\citep{lee2006riemannian}.
\end{definition}
\begin{definition}
A Riemannian metric for a smooth manifold $\gM$ is a bilinear, symmetric, positive definite map $\GBRV_x:\gT_{x}\gM\times \gT_{\xBRV}\gM\rightarrow\sR$ for all $\xBRV\in \gM$, where $T_{\xBRV}\gM$ is a tangent plane at point $x$ on the manifold~\citep{lee2006riemannian}.
\end{definition}
\begin{definition}
A smooth manifold $\gM$ with a Riemannian metric $\GBRV$ defined on every point of a manifold is called a Riemannian Manifold~\citep{lee2006riemannian}.
\end{definition}
\begin{definition}
Given a mapping $f:\gM\rightarrow\gN$ from smooth manifold $\gM$ to $\gN$, for any $\xBRV\in \gM$ the pull-back metric $\GBRV_{\xBRV}$ induced by the mapping $f$ is given as $\GBRV_{\xBRV} = \JBRV_{f(\xBRV)}^T \GBRV_{f(\xBRV)} \JBRV_{f(\xBRV)}$, where $\JBRV_{f(\xBRV)} = \frac{\partial f(\xBRV)}{\partial \xBRV}$.
\end{definition}

\label{sec:method}
\subsection{Latent space distortion}
To understand the latent distortion, we use an alternative derivation of ELBO~\citep{makhzani2015adversarial},
\begin{equation}
- RE[\log p_\theta(\xBRV| \zBRV)] - CE [q_\phi(\zBRV)||p(\zBRV)] + H[q_{\phi} (\zBRV| \xBRV)]   
\label{eq:elbo_vae1}    
\end{equation}
where $RE$ is the reconstruction error, $CE$ is the cross entropy, and $H$ is the entropy. For a Gaussian encoder, the $H$ term maximises the variance of latent encoding, the $RE$ term forces the distribution to be peakier (low variance), and the $CE$ term is to force the approximated posterior to match with the prior distribution. Optimizing the full objective means the encoder has to match a fixed prior and simultaneously have a peaky distribution. Due to this tradeoff, the encoder fails to match the fixed prior, resulting in the \textit{posterior-prior gap}. As a result, latent space is distorted with regions of low or zero density for which the decoder is unconstrained. For the derivation of Equation~\ref{eq:elbo_vae1} with a more comprehensive discussion, we refer readers to~\citep{makhzani2015adversarial, tomczak2022deep}. 

The \textit{posterior-prior gap} implies small changes in the input sample can result in significant changes in the latent encoding leading to abrupt changes in the reconstruction. Figure~\ref{fig:manifoldmap} on the left demonstrates an example of a vulnerability of VAEs. 

\subsection{Related work}
The distortions in the latent space of VAEs have been previously shown to result in limited generalisation capacity of VAEs~\citep{rezende2018taming, chen2020learning}.~\citet{rezende2018taming} further proposed a constrained optimisation to control the model performance. In our work, we take a geometric view of such distortion and use it to investigate the robustness of VAEs.

The use of encoder Jacobian matrices has previously appeared in approximating a tangent space of a data manifold, which captures the data points' sensitivity to its latent encoding~\citep{salah2011contractive, rifai2011manifold}. However, they don't consider the stochasticity of the encoder and decoder mappings. Several recent works treat the decoder mapping of VAEs as a smooth immersion and use the pullback as an induced metric in the latent space. The computation of such metrics has been helpful in various applications such as drawing on manifold samples, latent space interpolation, clustering, motion planning and many more~\citep{arvanitidis2017latent,yang2018geodesic, hauberg2018only,chen2018metrics, shao2018riemannian, arvanitidis2020geometrically,beik2022reactive}. 
~\citep{chen2020learning} propose to flatten the pullback metric tensor in the latent space induced by the decoder, which allows them to use standard Euclidean distance as a metric in latent space. The computation of a pullback in~\citet{chen2018metrics,yang2018geodesic,chen2020learning} does not consider the contribution of the uncertainty in the decoder mapping. It is, therefore, limited in its ability to capture the topological properties of the manifold. ~\citet{arvanitidis2017latent, hauberg2018only} proposed to consider the uncertainty of the decoder by treating the Gaussian decoder as a random projection of a deterministic manifold. This viewpoint allows them to treat the reconstruction space as a random manifold and the pullback metric tensors as stochastic, proving helpful in handling topological holes and low-density regions.

We want to remark much of the existing methods focus on improving the sampling in the latent space or improving the performance on tasks. Here, we show such distortions can be exploited by an adversary and discuss the importance of local geometry for evaluating and training robust VAEs. Previously~\citet{zhao2019adversarial, sun2020fisher, martin2020inspecting} studied the spectrum of Fisher information (pullback from probability simplex to the input space) of a classifier to investigate the robustness of adversarial perturbations. However, there is no such study for generative models to our knowledge. Also, the metric tensor considered in our paper considers the effect of uncertainty in the latent space, which is vital for understanding the latent distortions.
% Our paper takes a manifold view of latent space, we show that an adversary can find an optimal attack by moving along the eigendirection of a pullback metric tensor induced by the encoder and decoder mappings. Moreover, we establish the connection between the directional bias of an encoder's stochastic pullback metric tensor and the robustness resulting from the $\beta$ parameter of $\beta-$VAE. 
\section{ADVERSARIAL ATTACK EXPLOIT LOCAL GEOMETRY}
Consider an encoder $f_{\phi}$ as a smooth immersion from a data manifold to a latent manifold. We can then utilise the pullback metric tensor induced by $f_{\phi}$ to express the infinitesimal distance in the input space in terms of the local metric tensor of the latent space. Thus, unlike the existing methods that rely on Euclidean distance in the input space, we use the pullback metric tensor to measure the infinitesimal distance. 
%Figure~\ref{fig:manifoldmap} shows an example of a pullback metric tensor that uses Jacobian of mapping f to estimate the local metric tensor.

Here, we first express the adversarial optimisation problem in terms of the pullback metric tensor. Next, we show that the adversary can exploit the directional bias of a metric tensor to design optimal attacks. 
\begin{remark}
The infinitesimal distance between the representation of any data point $\xBRV$ and its perturbation $\xBRV_{\eta}=\xBRV+\etaB$ under parametric encoder $f_{\theta}$ for $\etaB$ small in $l_2$ is approximated as,\newline
i) $d(f_{\theta}(\xBRV), f_{\theta}(\xBRV+\etaB))=\etaB^T \GBRV_{\xBRV} \etaB$, where $\GBRV_{\xBRV} = \JBRV_{f_{\theta}(\xBRV)}^T \JBRV_{f_{\theta}(\xBRV)}$ for locally flat latent manifold, \newline
ii) $d(f_{\theta}(\xBRV), f_{\theta}(\xBRV+\etaB))=\etaB^T \GBRV_{\xBRV} \etaB$, where $\GBRV_{\xBRV} = \JBRV_{f_{\theta}(\xBRV)}^T \GBRV_{\zBRV}\JBRV_{f_{\theta}(\xBRV)}$ for a latent Riemannian manifold equipped with metric tensor $\GBRV_{\zBRV}$.  Here $\GBRV_{\zBRV}=\JBRV_{g_{\phi}(\zBRV)}^T \JBRV_{g_{\phi}(\zBRV)}$ is a pullback under the parametric decoder mapping $g_\phi$.
\end{remark}
%\textbf{Infinitesimal distance}
%Since perturbation $\etaB$ is small in norm, we can approximate the distance $d(.,.)$ in Equation~\ref{eq:optim} as an infinitesimal distance along the latent manifold using the Taylor expansion of squared distance,
% \begin{align}
%     d(f_{\theta}(\xBRV), f_{\theta}(\xBRV+\etaB)) &= ||(f_{\theta}(\xBRV) - f_{\theta}(\xBRV+\etaB)||_2^2 = \etaB^T \GBRV_{\xBRV} \etaB, \nonumber \\
%     \GBRV_{\xBRV} &= \JBRV_{f_{\theta}(\xBRV)}^T \JBRV_{f_{\theta}(\xBRV)}, \quad \JBRV_{f_{\theta}(\xBRV)} = \frac{\partial f_{\theta}}{\partial \xBRV}     
% \end{align}
where $\JBRV_{f_{\theta}(\xBRV)}\in \mathbb{R}^{d_z\times N}$ is a Jacobian matrix, $d_z$ is the dimensionality of $\ZBRV$ and $N$ is the dimensionality of $\XBRV$.
The matrix $\GBRV_{\xBRV}$ is a symmetric, positive definite matrix known as a pullback metric tensor under the mapping $f_{\theta}$. We can use it to measure the local inner product for every $\xBRV$ in the input space $\xBRV\in \XBRV$. 

\subsection{Stochastic pull-back metric tensor} 

\citet{eklund2019expected, arvanitidis2017latent} proposed an expected pullback metric tensor for measuring distances in the latent space of deep generative models. Here we utilise the pullback of an encoder mapping to evaluate the robustness of representations. Consider a stochastic encoder mapping $f_{\theta}$ expressed as a combination of mean $\muB_{\theta}(\xBRV)$ and standard deviation $\sigmaB_{\theta} (\xBRV)$ parameterisations: $f_{\theta}(\xBRV) = \muB_{\theta}(\xBRV) + \epsilonB\odot \sigmaB_{\theta}(\xBRV)$,  where $\epsilonB\sim\gN(0,\IBRV_d)$, then a Jacobian of $f_{\theta}$ with respect to input $\xBRV$ can be expressed as $\JBRV_{f_\theta} = \JBRV_{\muB(\xBRV)} + \epsilonB \odot \JBRV_{\sigmaB(\xBRV)}$ and the pullback matrix $\GBRV_{\xBRV}$ as,
\begin{align}
 \GBRV_{\xBRV} =  (\JBRV_{\muB(\xBRV)} &+ \epsilonB \odot \JBRV_{\sigmaB(\xBRV)})^T  (\JBRV_{\muB(\xBRV)} + \epsilonB \odot \JBRV_{\sigmaB(\xBRV)}) \nonumber \\ 
 =\JBRV_{\muB(\xBRV)}^T \JBRV_{\muB(\xBRV)} &+ \JBRV_{\muB(\xBRV)}^T\epsilonB \JBRV_{\muB(\xBRV)} + \JBRV_{\sigmaB(\xBRV)}^T\epsilonB \JBRV_{\muB(\xBRV)} + \JBRV_{\sigmaB(\xBRV)}^T\epsilonB^2\JBRV_{\sigmaB(\xBRV)}\nonumber
\end{align}
We can view the latent space as a random projection of a deterministic manifold. Under the assumption the sample paths from $f_{\theta}$ are smooth, we can treat the metric tensor as a stochastic matrix. The metric of the random latent manifold can be estimated in expectation as $\hat{\GBRV}_{\xBRV} = \mathbb{E}_{\epsilonB\sim {p(\epsilonB})} [\GBRV_{\xBRV}]$. Since, $\epsilonB$ is a zero mean and unit covariance the $\mathbb{E}[\epsilonB]=0$ and $\mathbb{E}[\epsilonB^2]=1$ the final expected metric tensor is,
\begin{equation}
\label{eq:stochasticpull}
 \hat{\GBRV}_{\xBRV} =  \JBRV_{\muB(\xBRV)}^T\JBRV_{\muB(\xBRV)}  + \JBRV_{\sigmaB(\xBRV)}^T\JBRV_{\sigmaB(\xBRV)}
\end{equation}

\begin{theorem} Given a stochastic encoder mapping $f_{\theta}$, for an arbitrary data point $\xBRV\in \XBRV$ the adversarial perturbation $\xBRV_{\eta}$ under $l_2$ norm is optimal when moving along the eigendirection of a stochastic pullback metric tensor induced by $f_{\theta}$ at $\xBRV$.
\end{theorem}
\begin{proof}
By using the stochastic pull back metric tensor given in Equation~\ref{eq:stochasticpull}, we reformulate the adversarial attack optimisation of Equation~\ref{eq:optim} as,
\begin{align}
\label{eq:optim}
\max_{\etaB} & \quad \etaB^T \hat{\GBRV}_{\xBRV} \etaB\nonumber\\
&\qquad \qquad \text{subject to} \quad ||\etaB||_2 = \eta_0
\end{align}
Next, to solve the problem, we combine the constraints by introducing Lagrange multiplier $\lambdaB$,
\begin{equation}
\label{eq:optim3}
\max_{\etaB} \quad \etaB^T \hat{\GBRV}_{\xBRV} \etaB + \lambdaB ( \eta_0 - ||\etaB||_2^2)
\end{equation}
The closed-form solution of the above optimisation takes the form:
%$\hat{\GBRV}_{\xBRV} \etaB= \lambdaB \etaB$,
 \begin{equation}
     \hat{\GBRV}_{\xBRV} \etaB= \lambdaB \etaB
\end{equation}
thus a pair $(\lambdaB, \etaB)$  represents the eigenvalue and eigenvector of the stochastic pullback metric tensor, the eigenvector with the largest eigenvalue corresponds to the direction of maximal change. 
\end{proof}

\begin{figure}[t]
    \centering
    \small{(a)  $\beta=0.01$ \qquad \qquad (b) $\beta=1.0$ \qquad \qquad (c) $\beta=7.5$}\\
    \includegraphics[width=0.32\linewidth]{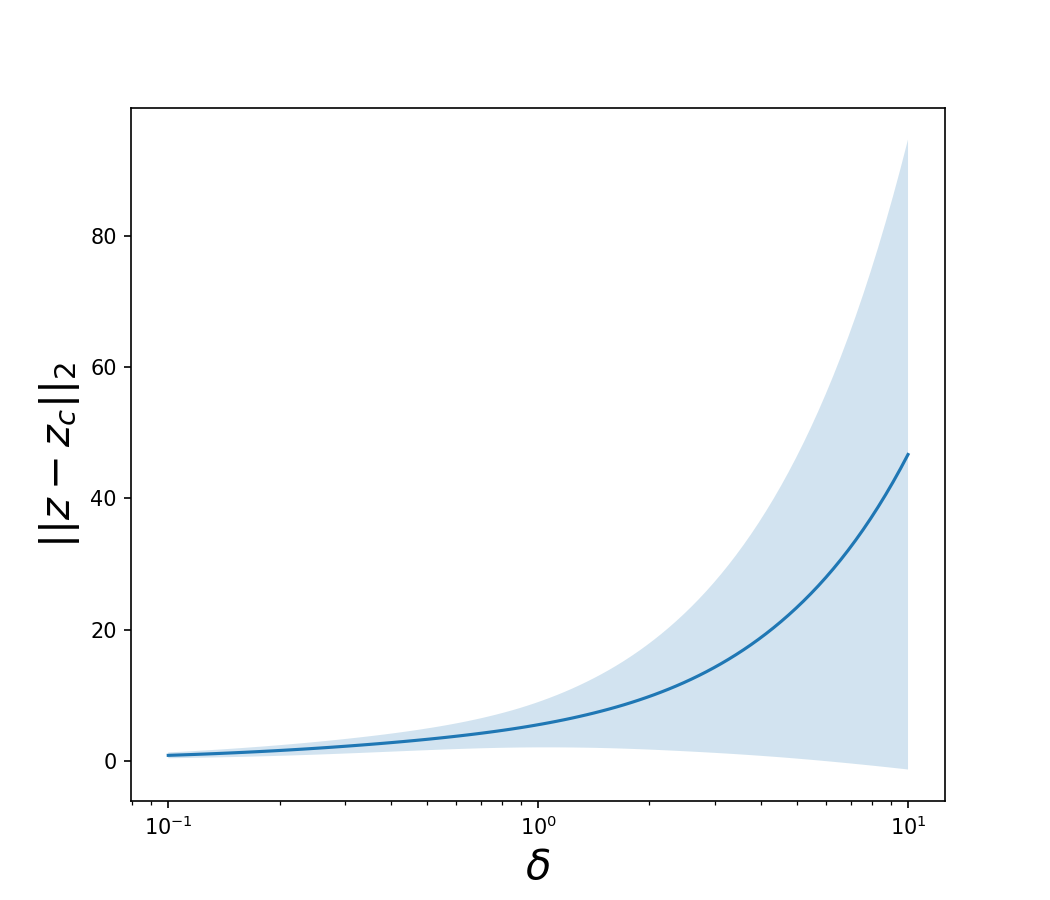}
    \includegraphics[width=0.32\linewidth]{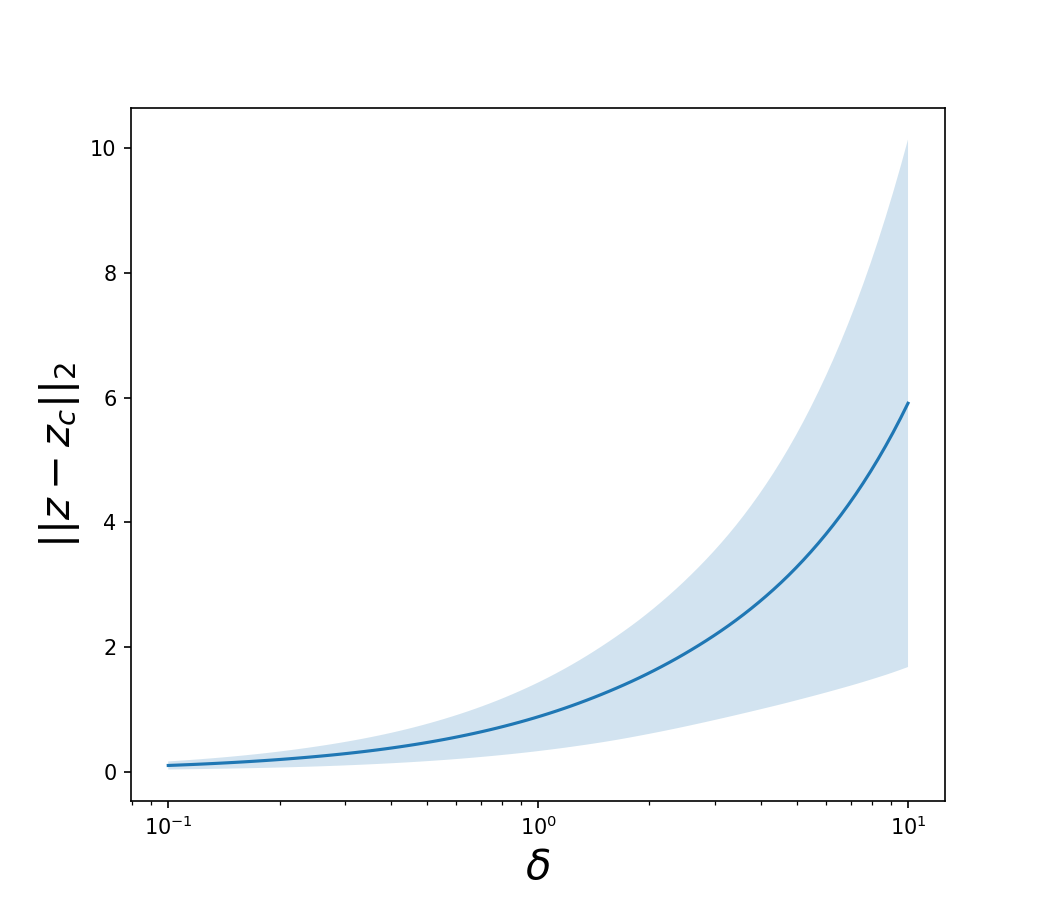}
    \includegraphics[width=0.32\linewidth]{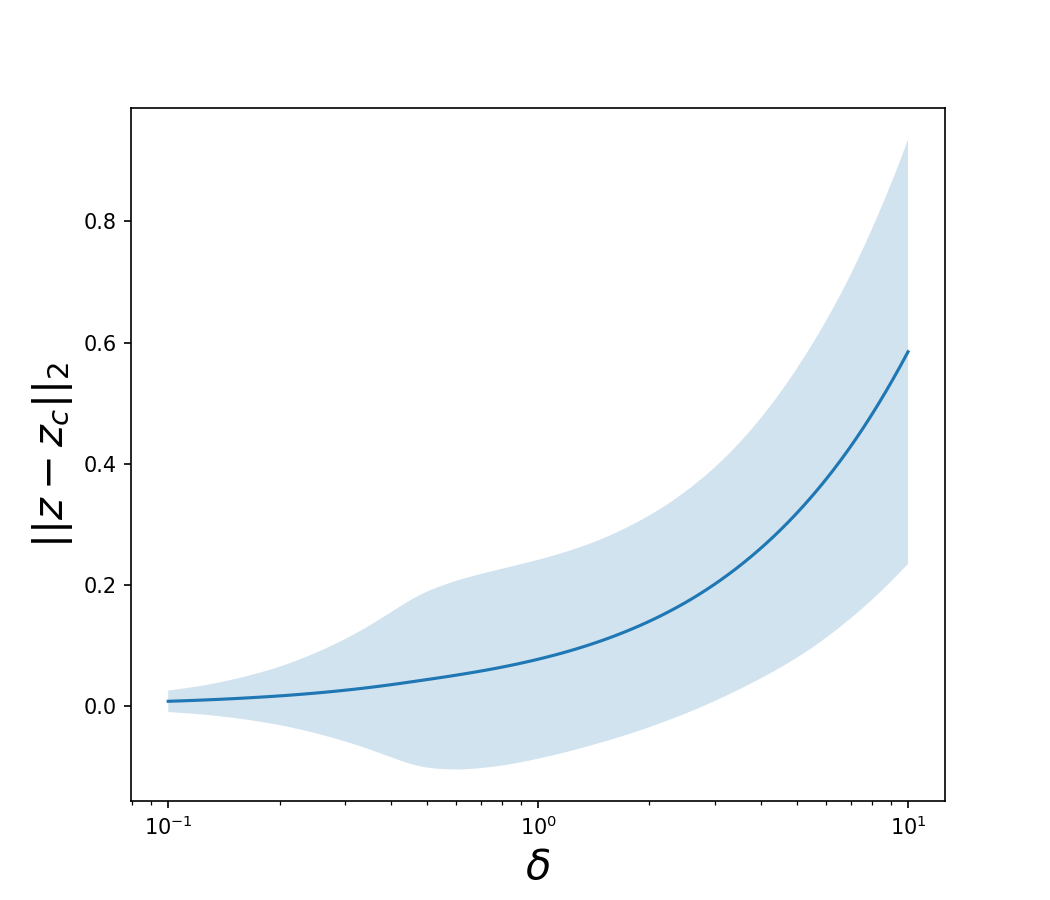}
    \caption{The plot shows the change in the latent encoding of $\beta-$VAE (in terms of Euclidean distance) for different values of $\beta$ when moving along the dominant eigendirection of a pullback metric tensor $\hat{\GBRV_{\xBRV}}$ with different step size $\delta$. We can see for small $\beta$, the changes are of much higher magnitude compared to larger $\beta$, demonstrating that increasing the $\beta$ makes the latent space more smooth.}
    \label{fig:latent_dist}
\end{figure}

Thus, for any given input $\xBRV$ an adversary can design an optimal attack by taking a step along the eigendirection of $\hat{\GBRV}_{\xBRV}$ as $\xBRV_{c} = \xBRV + \delta \lambdaB \etaB$, where $\delta$ is a step size. To compromise the reconstruction the step size $\delta$ can be chosen such that $||\xBRV - \hat{\xBRV}_{c}||_2>||\xBRV - \hat{\xBRV}||_2$, where $\hat{\xBRV}$ and $\hat{\xBRV}_{c}$ are reconstructions of original input and its corrupted version. 

\begin{figure*}[ht!]
\small{\textbf{CelebA}: (a) Original  $\downarrow$ Reconstruction  \qquad \qquad(b) Corrupted $\delta_1$ $\downarrow$ Reconstruction \qquad \qquad(c) Corrupted $\delta_2$ $\downarrow$ Reconstruction}
\begin{multicols}{3}
\includegraphics[height=1.85cm, width=5.25cm]{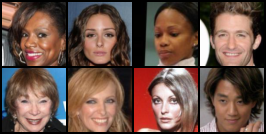}  \includegraphics[height=1.85cm, width=5.25cm]{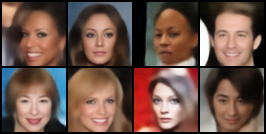}
\par
 \includegraphics[height=1.85cm, width=5.25cm]{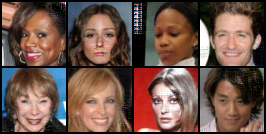}
\includegraphics[height=1.85cm, width=5.25cm]{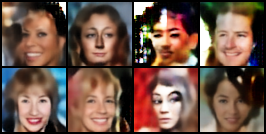} 
 \par
 \includegraphics[height=1.85cm, width=5.25cm]{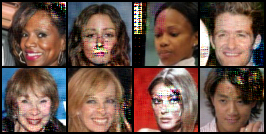}
 \includegraphics[height=1.85cm, width=5.25cm]{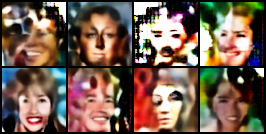}
\end{multicols}
\vspace{-0.25cm}
    \caption{Illustration of adversarial attack along the dominant eigenvector of a stochastic pullback metric tensor on CelebA dataset. We evaluate the reconstruction for original images and its two corrupted versions with different step sizes $\delta_1=0.5233$ and $\delta_2=0.7443$.}
    \label{fig:attackceleb}
\end{figure*}

\begin{figure*}[ht!]
    \centering
    \small{(a) Robustness evaluation of $\beta-$VAE on MNIST.}\\
    \includegraphics[width=0.24\linewidth]{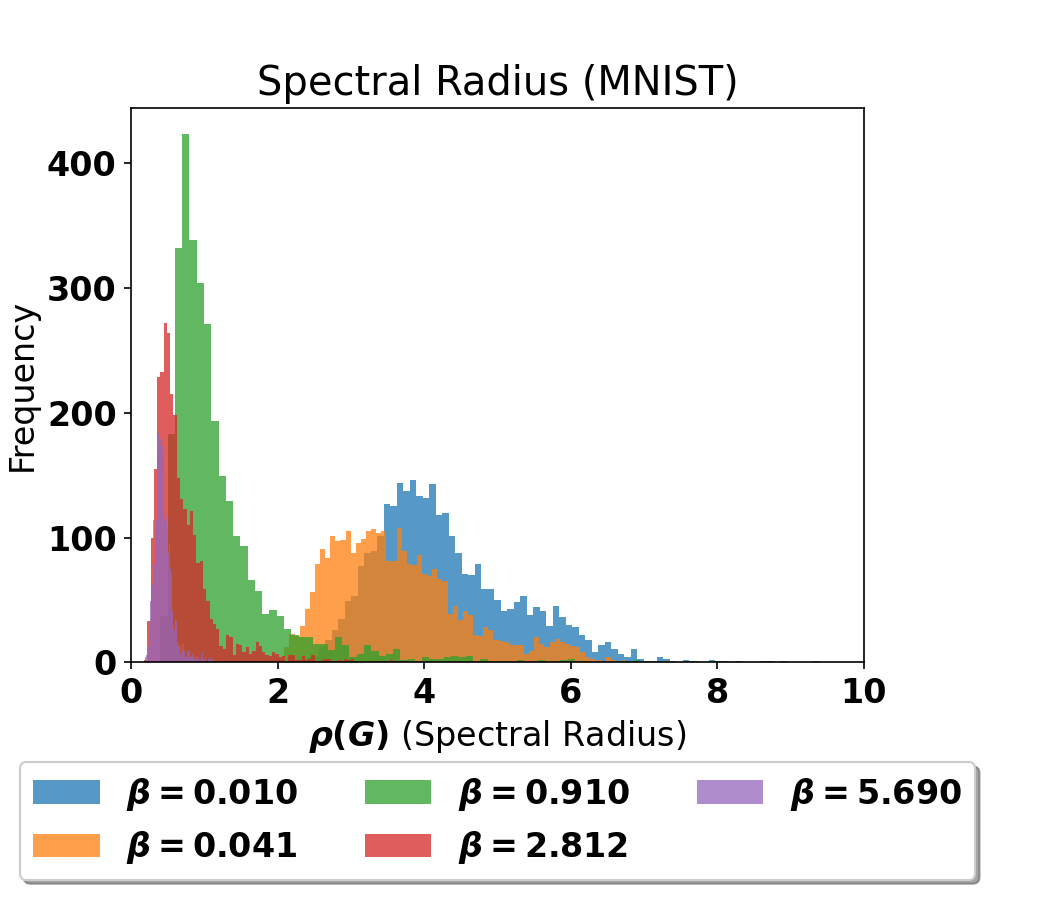}
    \includegraphics[width=0.24\linewidth]{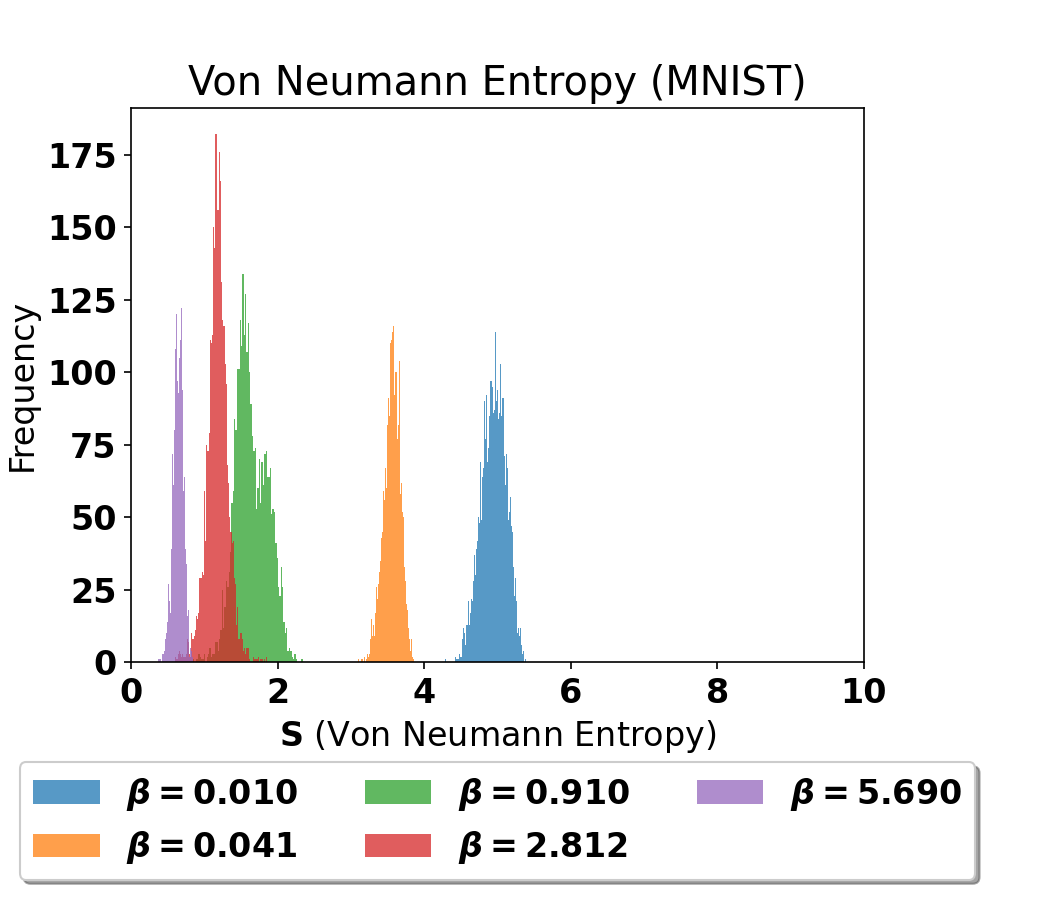}
    \includegraphics[width=0.24\linewidth]{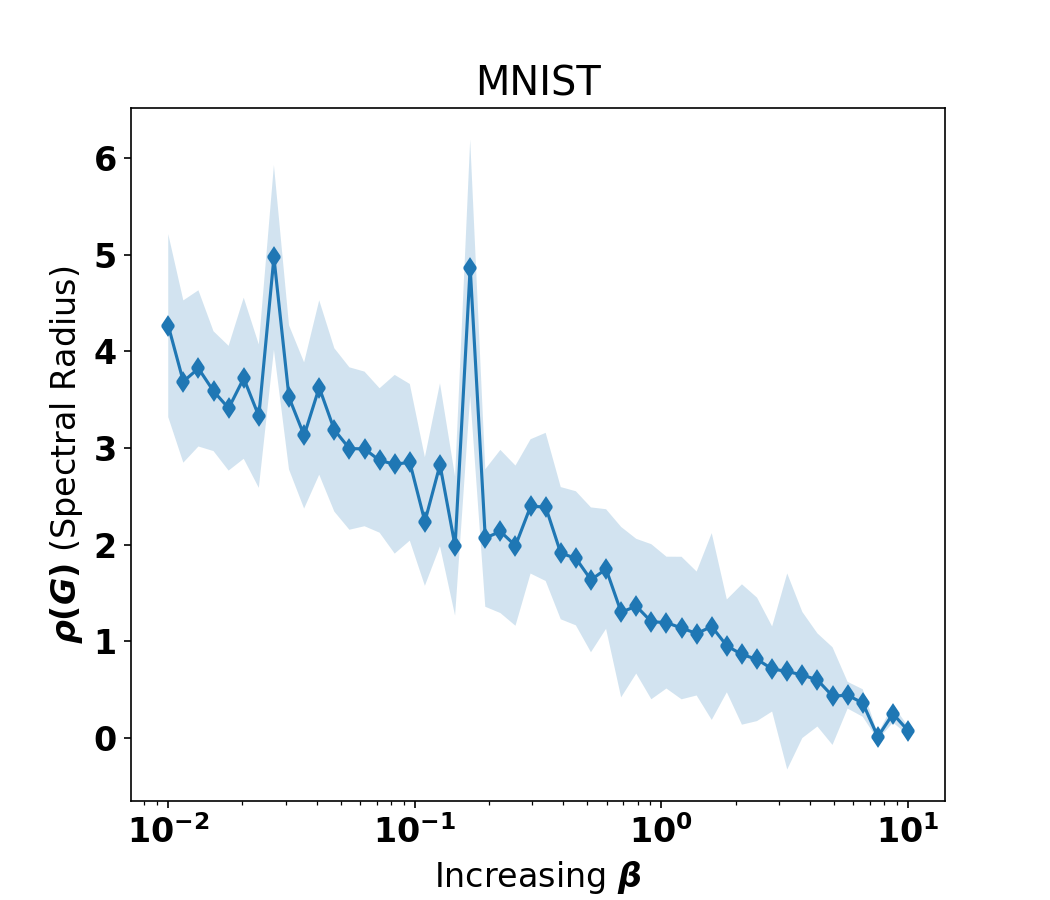}
    \includegraphics[width=0.24\linewidth]{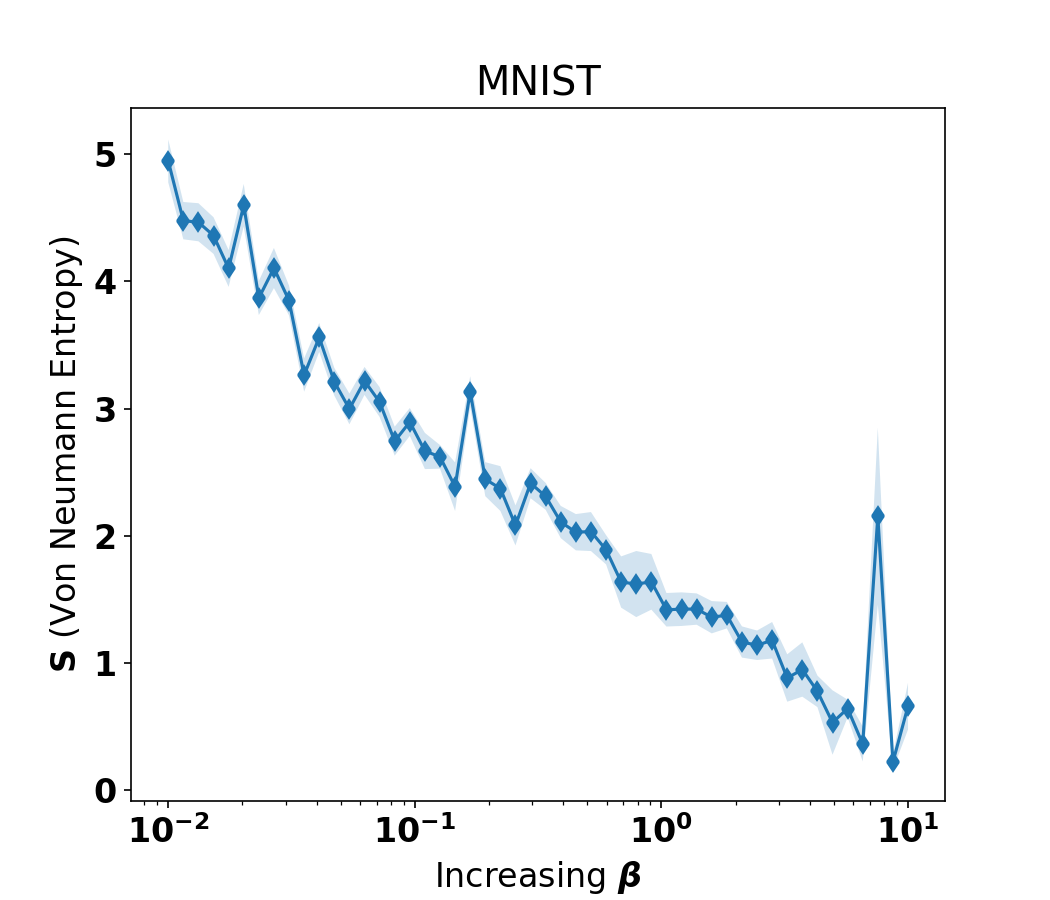}
    \includegraphics[width=0.24\linewidth]{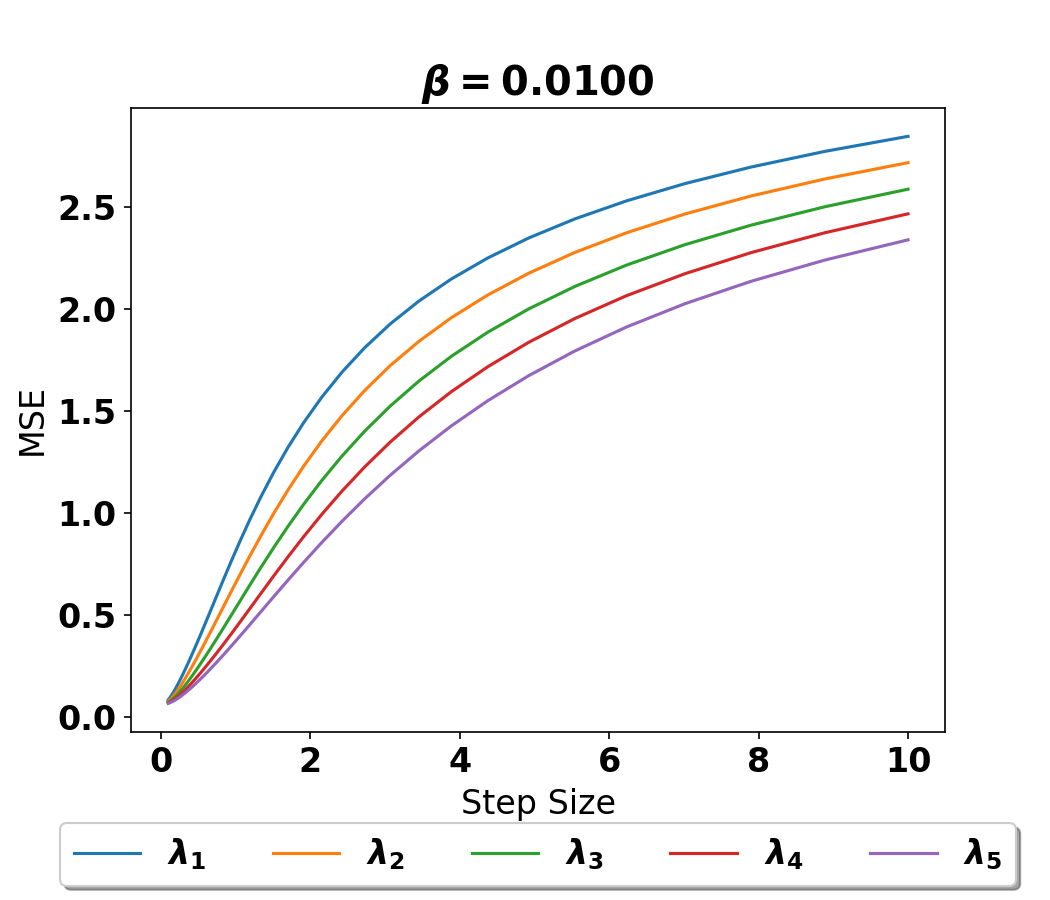}
    \includegraphics[width=0.24\linewidth]{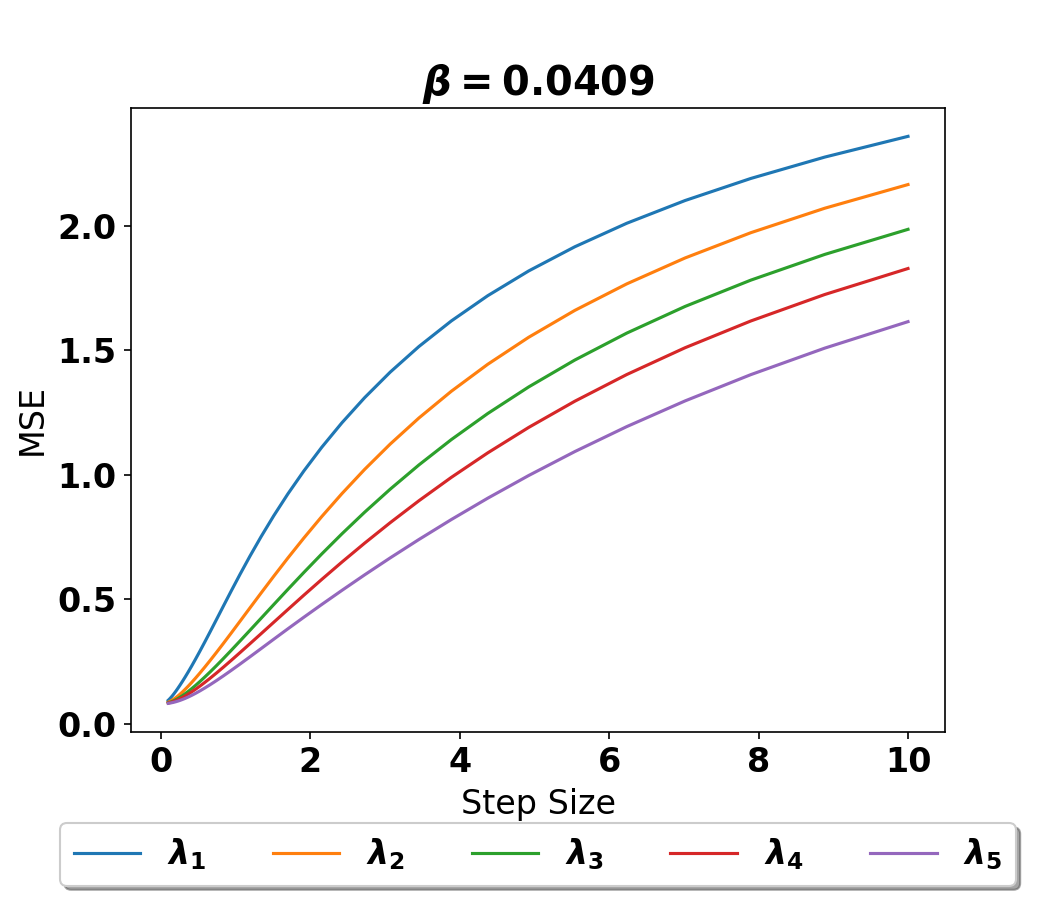}
    \includegraphics[width=0.24\linewidth]{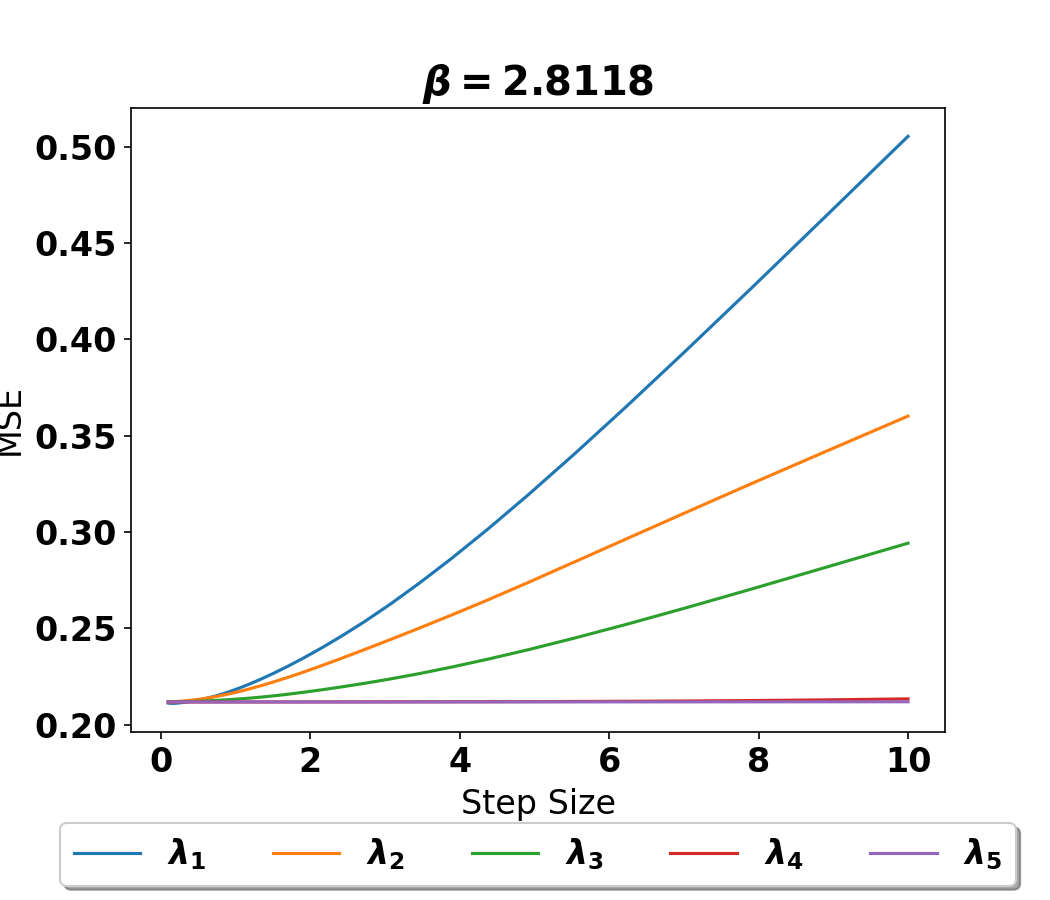}
    \includegraphics[width=0.24\linewidth]{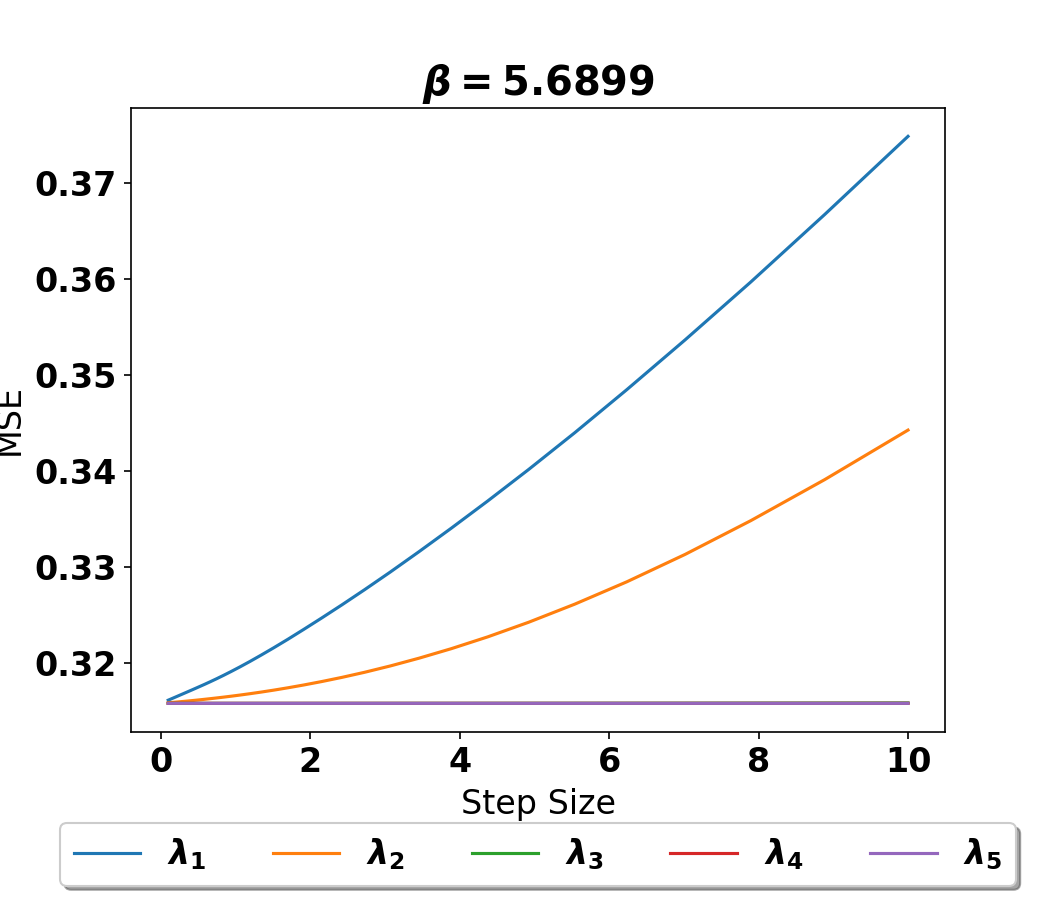}\\
    \small{(b) Robustness evaluation of $\beta-$VAE on FashionMNIST.}\\
    \includegraphics[width=0.24\linewidth]{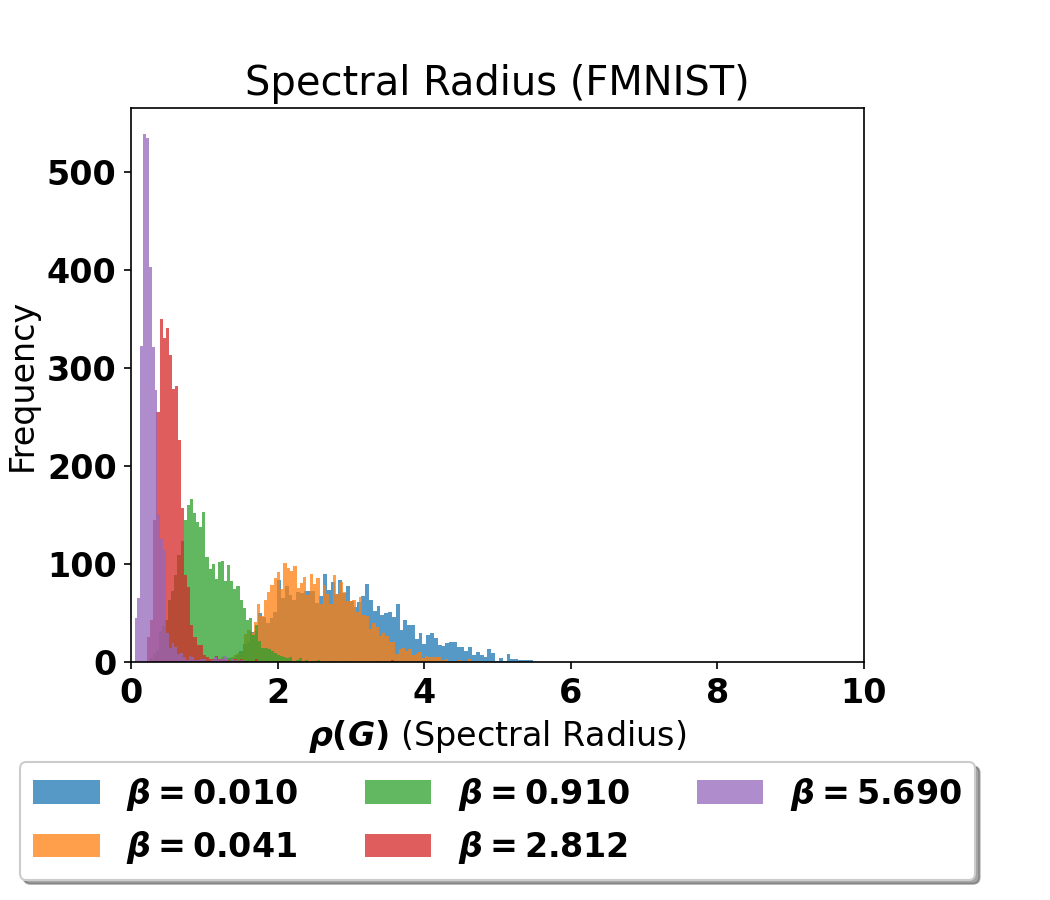}
    \includegraphics[width=0.24\linewidth]{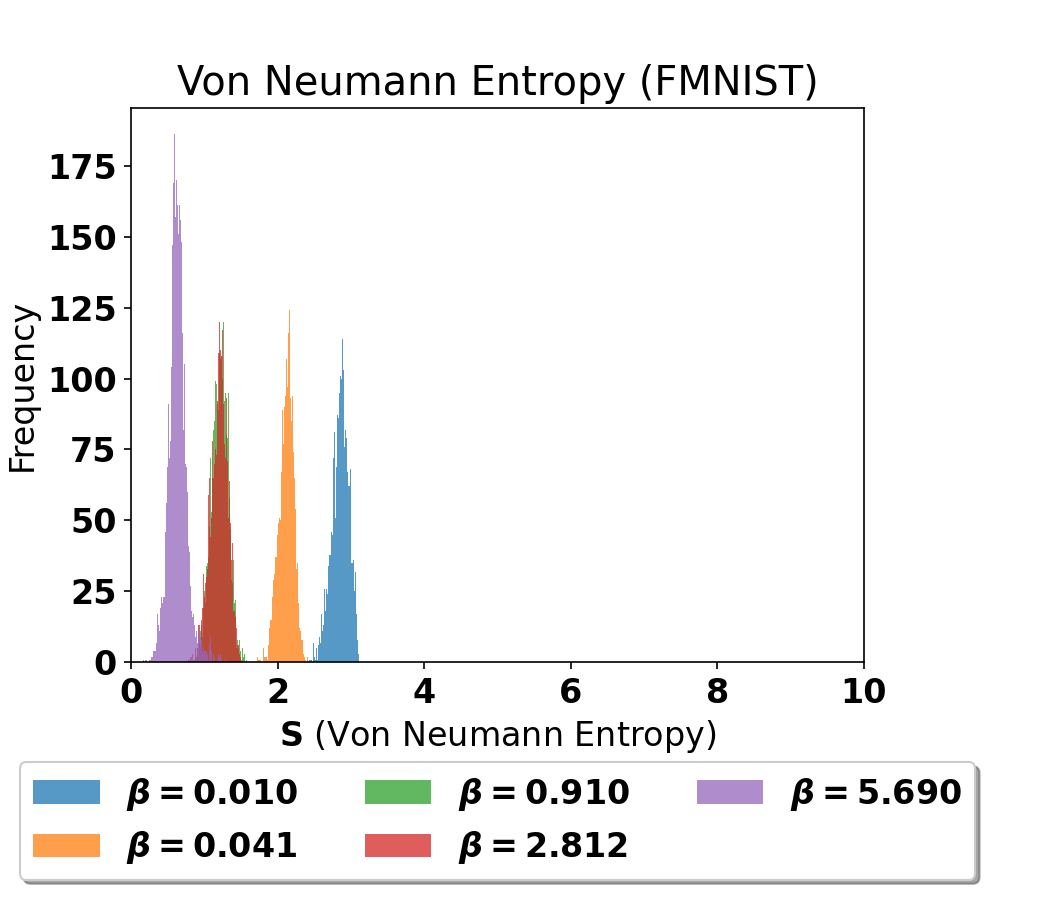}
    \includegraphics[width=0.24\linewidth]{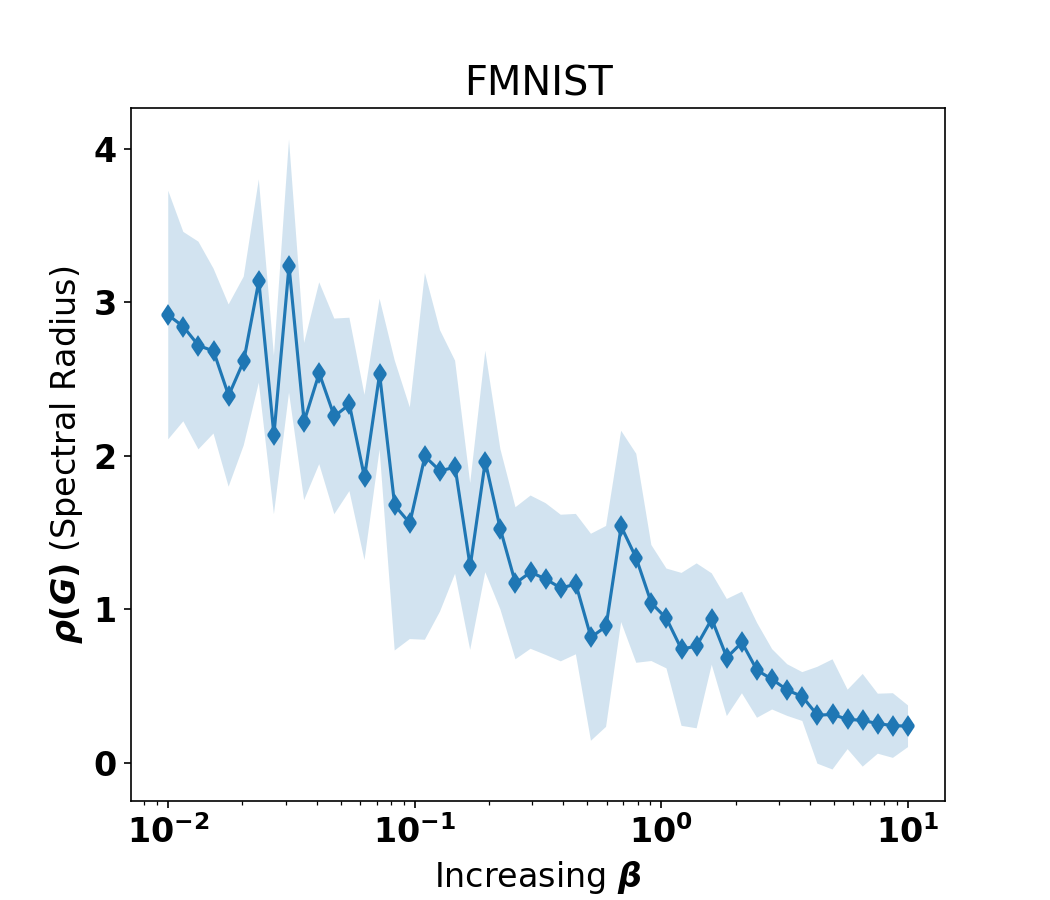}
    \includegraphics[width=0.24\linewidth]{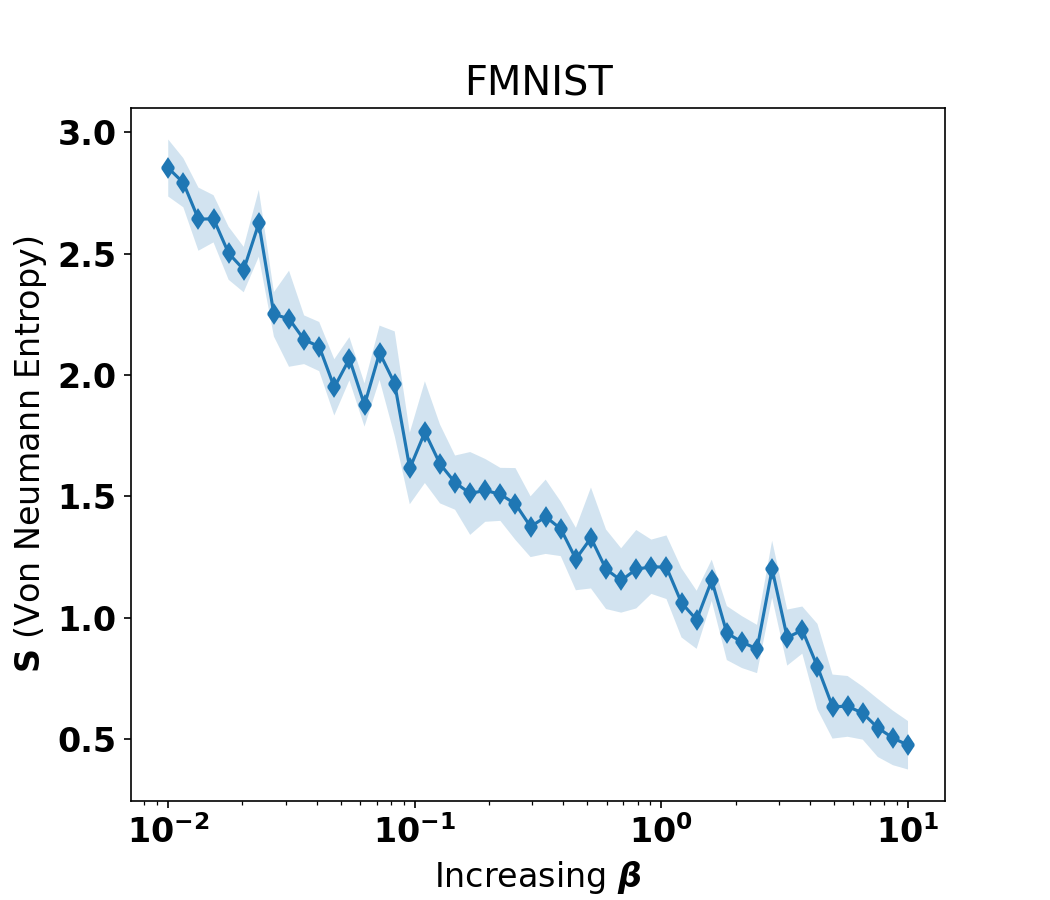}
    \includegraphics[width=0.24\linewidth]{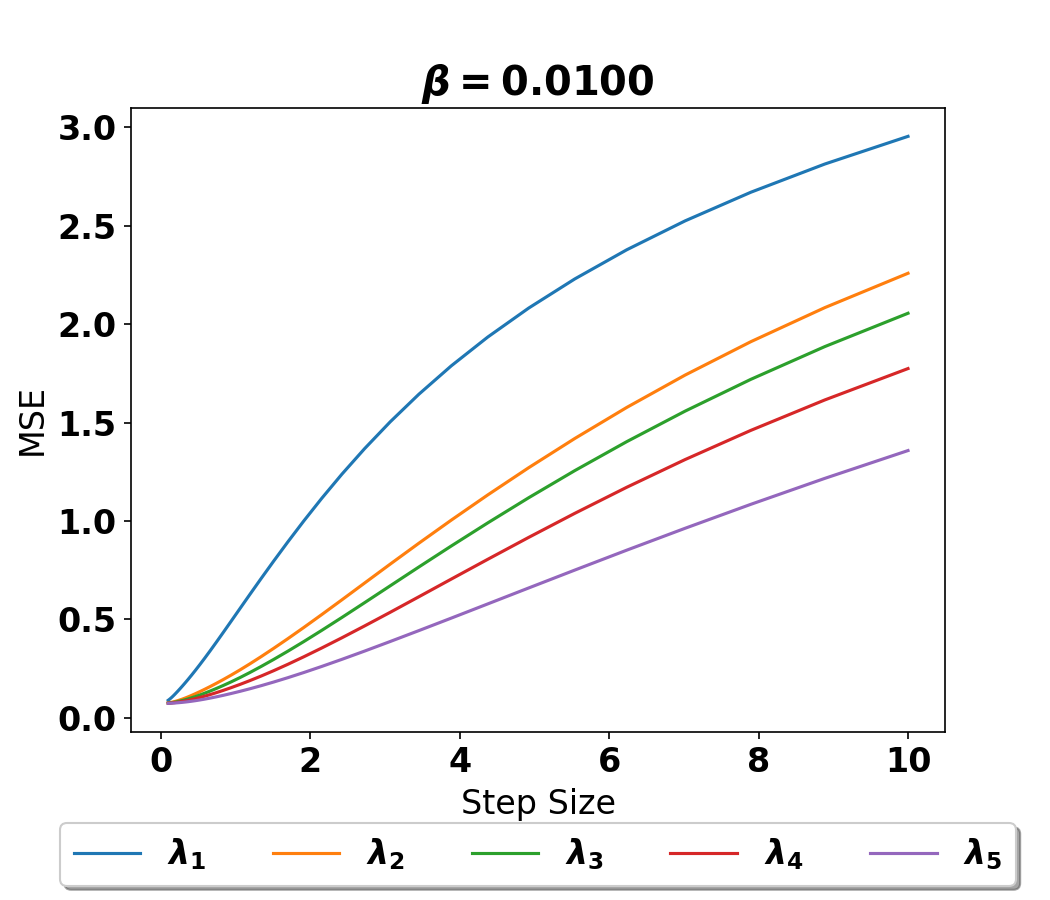}
    \includegraphics[width=0.24\linewidth]{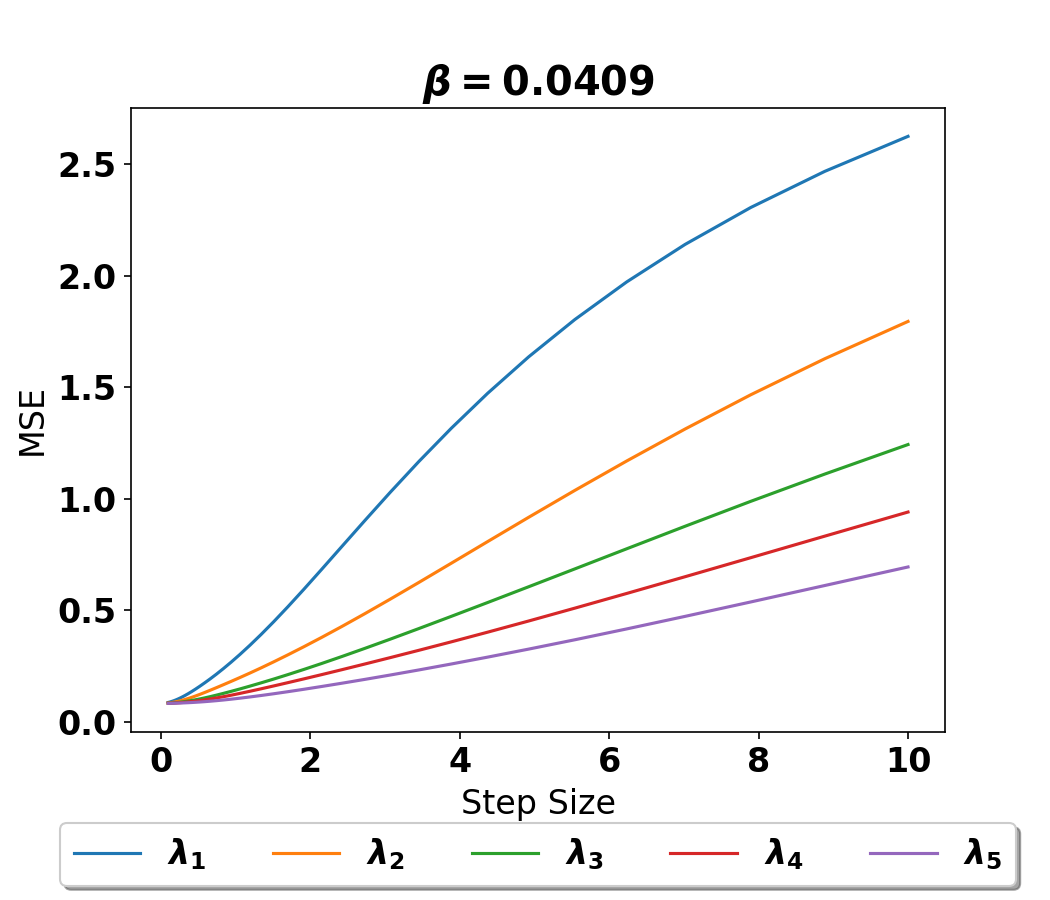}
    \includegraphics[width=0.24\linewidth]{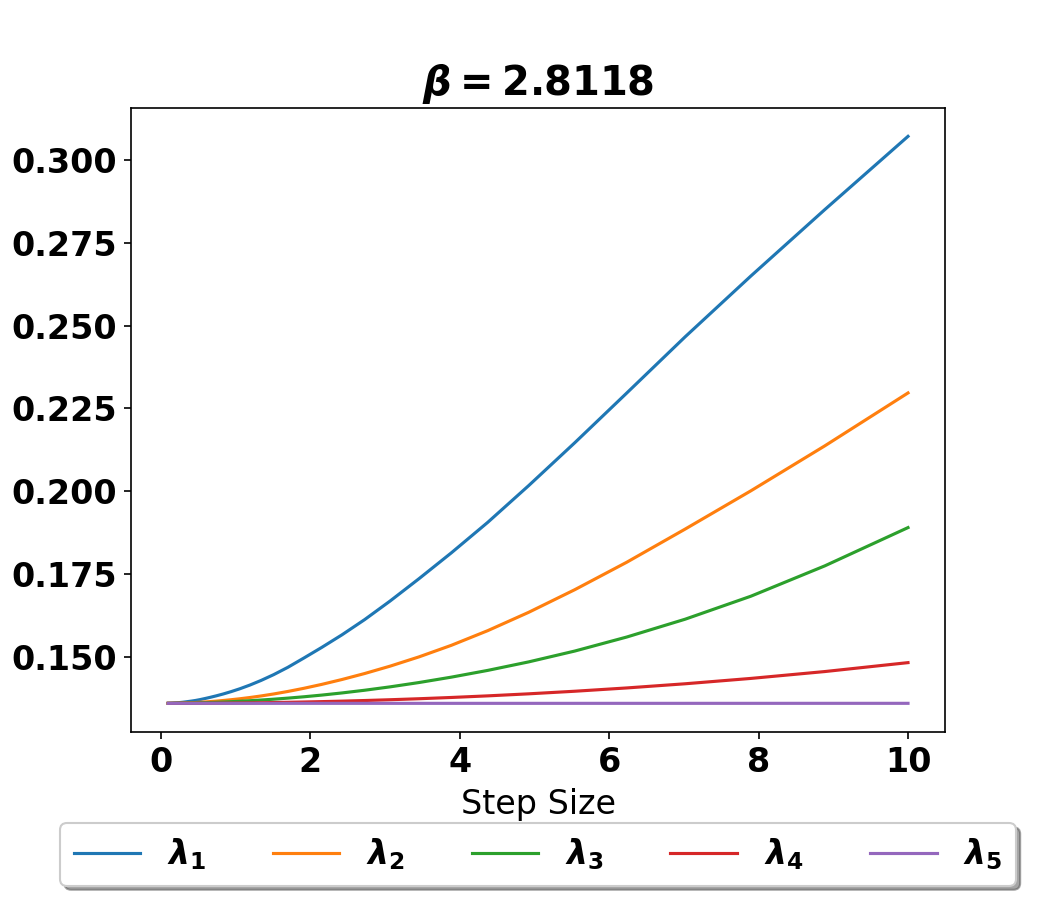}
    \includegraphics[width=0.24\linewidth]{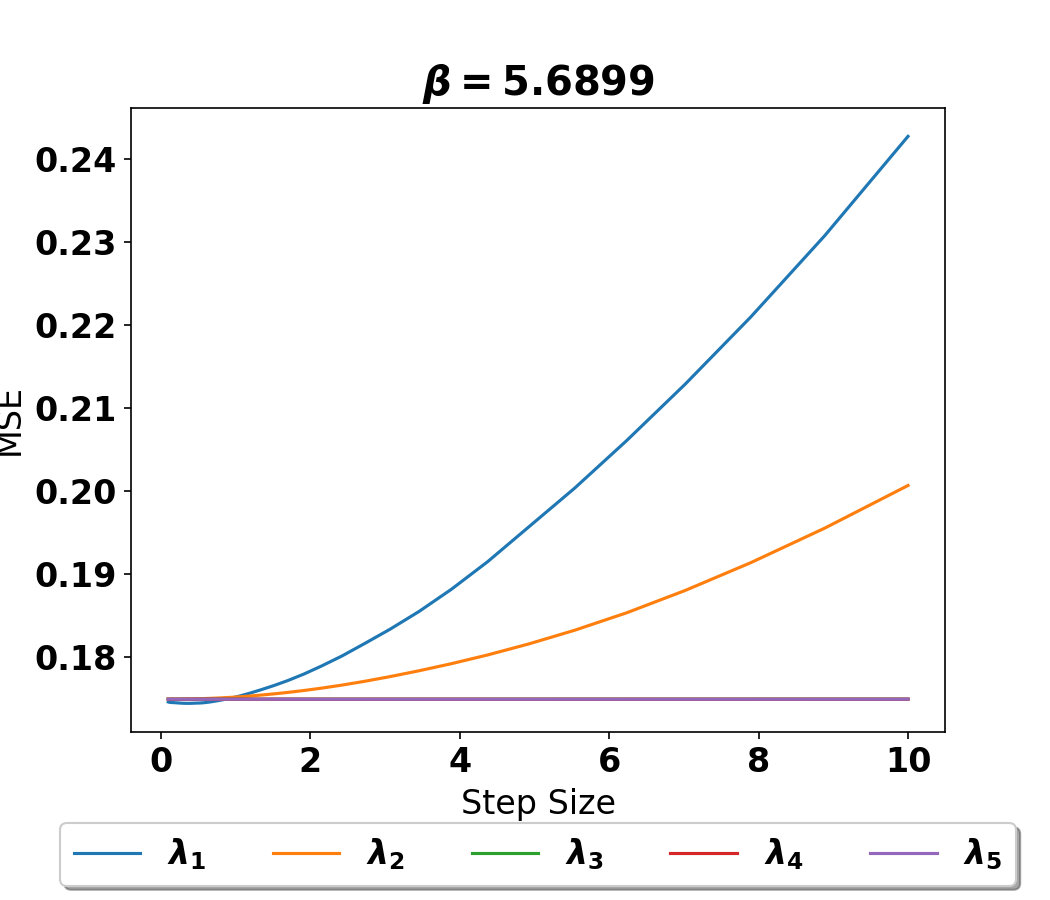}
    \caption{ Figure (a), on the left, we report the histogram of spectral radius and Von Neumann entropy (on test samples) for different values of $\beta$ in $\beta$-VAE. On the right, we report the average of two scores across test samples for an increasing value of $\beta$. We observe that increasing the value of $\beta$ suppresses the metric tensor's maximum eigenvalue, and the eigenspectrum distribution gets more isotropic. In the second row, we corrupt the test images along the top five eigendirections (denoted by $\lambda_1, \lambda_2, \lambda_3, \lambda_4, \text{ and } \lambda_5$) with an increasing step size for different values of $\beta$. The plots describe the average MSE across test samples. We observe that the average step size increases for a higher value of $\beta$. Increasing the value of $\beta$ reduces the \textit{posterior-prior gap}, minimising distortion in the latent space. Figure (b) demonstrates similar observations on the FashionMNIST dataset.}
    \label{fig:evaluation}
\end{figure*}

\subsection{Robustness evaluation}
A robustness method should suppress the maximum eigenvalue of the pullback metric tensor. Moreover, it should eliminate the directional bias resulting from the anisotropic distribution of eigenvalues of a pullback metric tensor. To quantify these two effects, we report the following two scores,

\textit{Spectral Radius} for a matrix $\GBRV$ is defined as, 
\begin{equation}
     \rho (\GBRV) = \max \{|\lambdaB|, \lambdaB \text{ is an eigenvalue of } \GBRV \}
\end{equation}
A robust model will have a low value of spectral radius. 

\textit{Von Neumann Entropy}~\citep{bengtsson2008geometry} $\SBRV$ of a metric tensor $\GBRV$ is given by the Shannon entropy of its eigenvalues $ \SBRV = - \sum_k \lambdaB_k \log \lambdaB_k$. 
% \begin{equation}
%     \SBRV = - \sum_k \lambdaB_k \log \lambdaB_k
% \end{equation}
The high value would imply the metric tensor is anisotropic, resulting in a directional bias. Thus, a robust model will have a low value of $\SBRV$.

\section{RESULTS AND DISCUSSION}
As pointed out in Remark (3.1), we can treat the latent space as (i) locally flat or (ii) a Riemannian manifold equipped with a metric tensor at all points. In the latter case, similar to the encoder, the stochastic pull-back metric tensor in the latent space can be computed as $\GBRV_{\zBRV}=  \JBRV_{g_{\omega}(\zBRV)}^T  \JBRV_{g_{\omega}(\zBRV)}$ where $g_{\theta}$ is a decoder network. The main paper presents case (i) results where latent space is $\GBRV_{\zBRV}=\IBRV$. We also analyze the case $\GBRV_{\zBRV}$ is a pullback induced by the decoder network; the results for this one are reported in Supplementary~\ref{apsec:extraresults}.
 
We used PyTorch~\citep{NEURIPS2019_9015} for the implementation of our work. The experimental details, including the training procedure, are discussed in Supplementary Section~\ref{implement}. The implementation is publicly available on GitHub\footnote{https://github.com/MdAsifKhan/RobustnessVAE/}. Here, we first empirically demonstrate the vulnerability of a VAE using a one-step attack along the dominant eigendirection of input samples. Next, we investigate the robustness of $\beta-$VAE and discuss a simple mixup strategy that fills a \textit{posterior-prior gap} in the latent space, flattens the latent space and ensures the decoder generated valid samples.

\subsection{Adversarial attack} Figure~\ref{fig:attack} demonstrates the two instances of corruption along the dominant eigenvector of $\beta=1$ VAE on MNIST~\citep{726791} and FashionMNIST~\citep{xiao2017fashion} datasets. For each dataset, the three columns in the first row are a set of original images and their corrupted version with a step size of $\delta=0.5223$ and $\delta=0.7443$. In the second row, we report their respective reconstructions. We observe that with $\delta=0.5223$, the reconstruction significantly differs from the original images, and for $\delta=0.7443$ gets much more severe, exposing the capacity of VAE. This result proves an attacker can exploit the metric tensor's directional bias to design a one-shot attack. Figure~\ref{fig:attackceleb} further demonstrates a similar finding on the CelebA~\citep{Liu_2015_ICCV} dataset.
\subsection{Robustness evaluation}
In this section, we first investigate the connection between our proposed scores and the robustness of $\beta-$VAE and later discuss an alternative robustness scheme using a mixup training loss. 
\subsubsection{$\beta-$VAE}
The $\beta$ parameter in a $\beta-$VAE controls the gap between an approximated posterior and a prior distribution. The high value of $\beta$ reduces the gap, thus eliminating the latent distortions which an adversary can exploit. We now investigate the effect of increasing values of $\beta$ on the two scores. We sample $50$ values of $\beta$ with a logarithmic spacing between $[0.01,10]$. We trained the encoder-decoder model for each parameter and computed the 
two scores $\rho (\hat{\GBRV)}$ and $\SBRV$ for every sample point.  Due to the high computational cost of training $50$ different models per dataset, in this section, we limit the experiments to MNIST and FashionMNIST datasets.

Figure~\ref{fig:evaluation} first and the second column in row one (MNIST) and in row three (FashionMNIST) reports the histogram of the scores for four different values of $\beta$. We observe that the higher value of $\beta$ suppresses the spectral radius. Similarly, the von Neumann entropy is decreased, demonstrating that the local directions get isotropic. Importantly, this indicates the adversary cannot exploit the directional bias for high values of $\beta$ with $\etaB$ small in the norm. In the third and fourth columns of rows one and three, we report the mean and standard deviation of the scores computed for fifty increasing values of $\beta$. The results demonstrate that by reducing the KL gap, parameter $\beta$ prevents distortion in the latent space eliminating the directional bias exploited by an adversary. 

Next, we examine the connection between the step size $\delta$ and the strength of attacks under different values of $\beta$. We report the mean squared error (MSE) between an original image and its reconstruction under varying corruption rates along five dominant eigendirections. We generated $40$ logarithmic spacing steps in the range $[0.01, 10]$. Figure~\ref{fig:evaluation} second row (MNIST) and fourth row (FashionMNIST) demonstrate the MSE vs step size averaged across all test samples for four different values of $\beta$. We observe that for the small $\beta$, all five directions tend to get high MSE, and as $\beta$ increases, it requires a larger step size to get a significant change in MSE. 

Further, to analyse the latent distortions, we generate perturbations of increasing magnitude for each sample $\xBRV$ by increasing step size $\delta$ along the dominant eigendirection of the pullback metric tensor. In Figure~\ref{fig:latent_dist}, we report the distance between the latent encoding of the original input and its perturbations averaged across data samples.
\begin{figure*}[t]
    \centering
    \small{Robustness evaluation of \textit{mixup} on MNIST and FashionMNIST.}\\
    \small{(a) MNIST \qquad \qquad \qquad \qquad \qquad \qquad \qquad \qquad \qquad (b) FashionMNIST}\\
\includegraphics[width=0.24\linewidth]{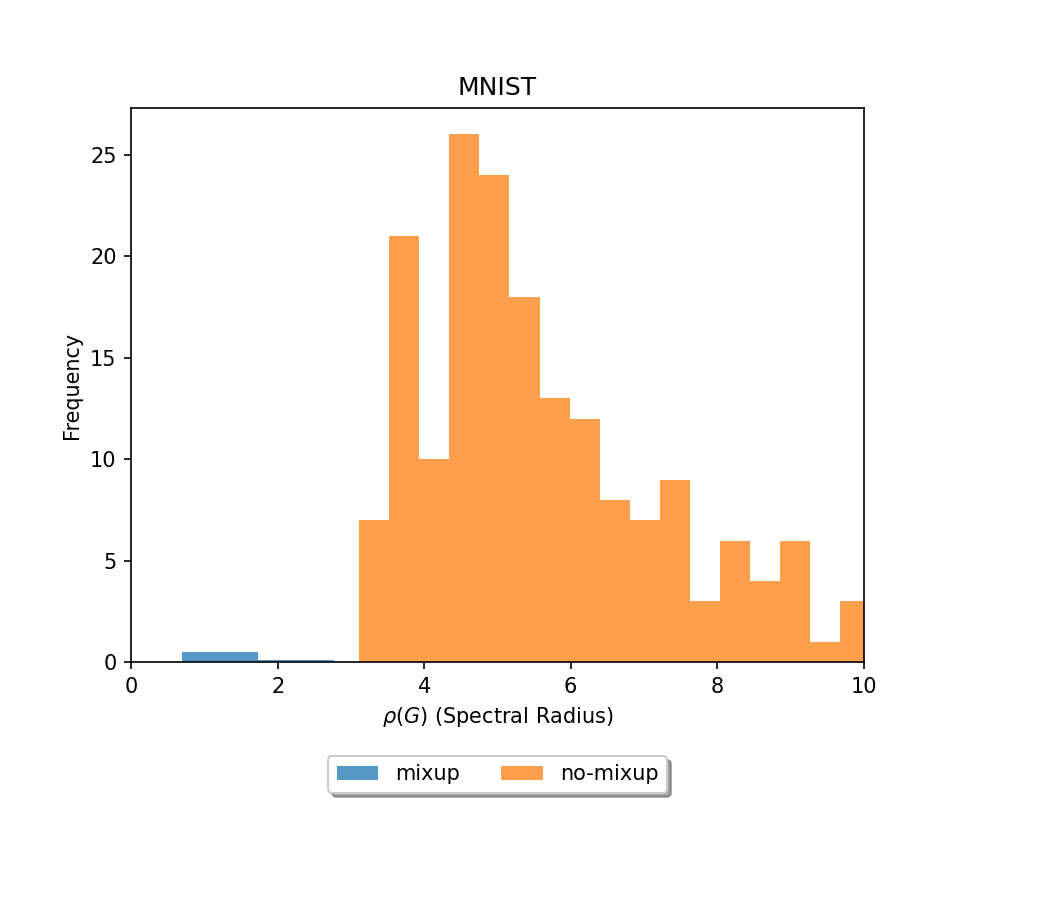}
    \includegraphics[width=0.24\linewidth]{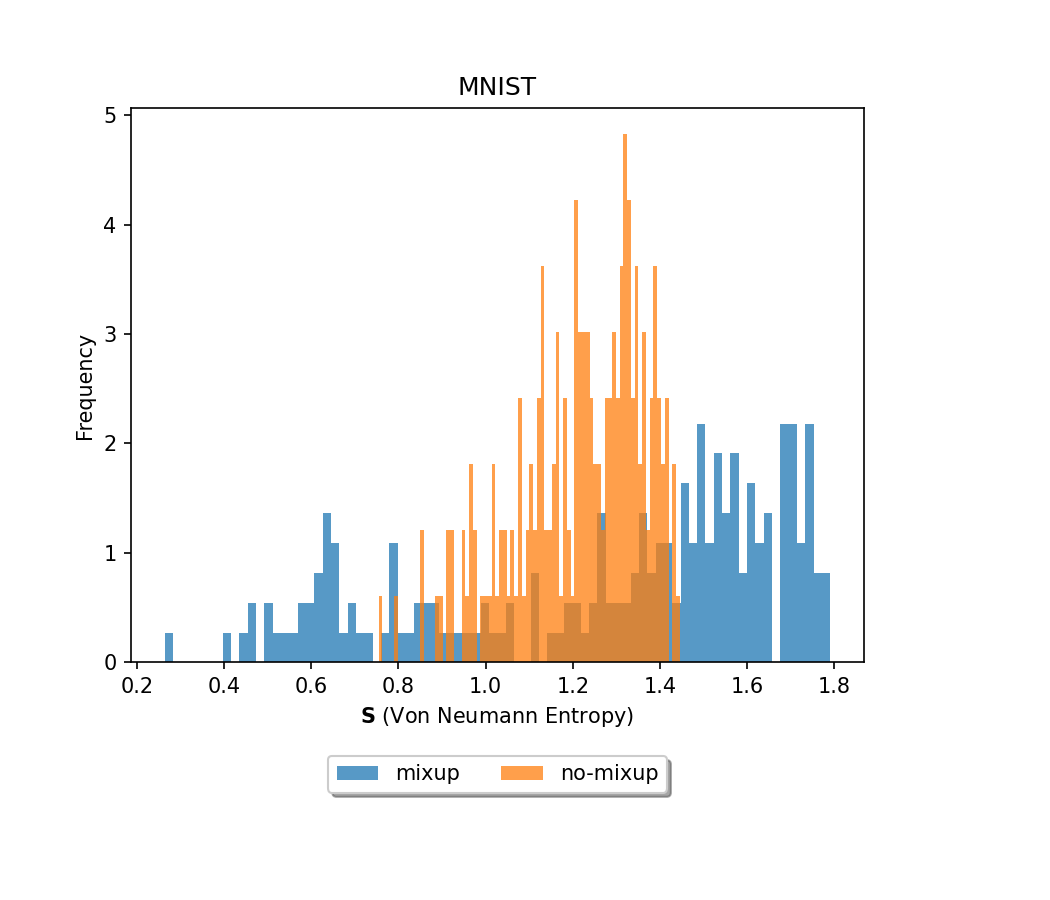}
    \includegraphics[width=0.24\linewidth]{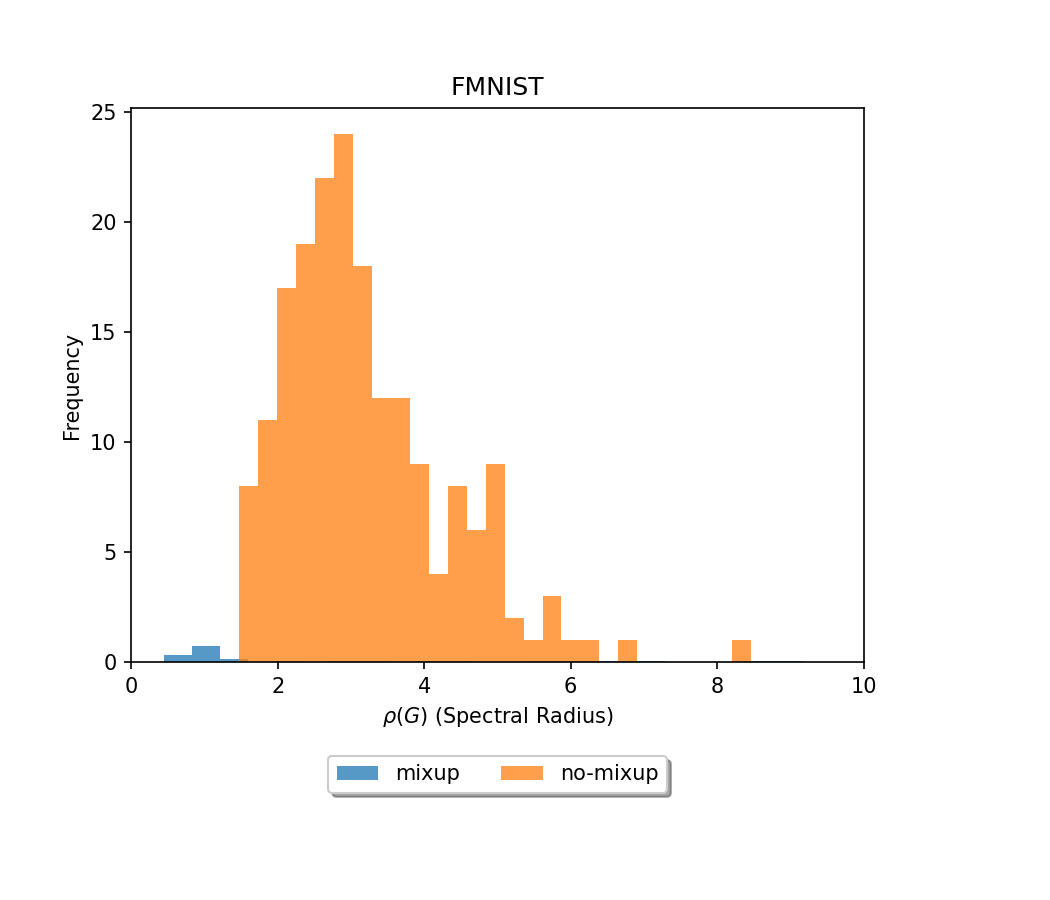}
    \includegraphics[width=0.24\linewidth]{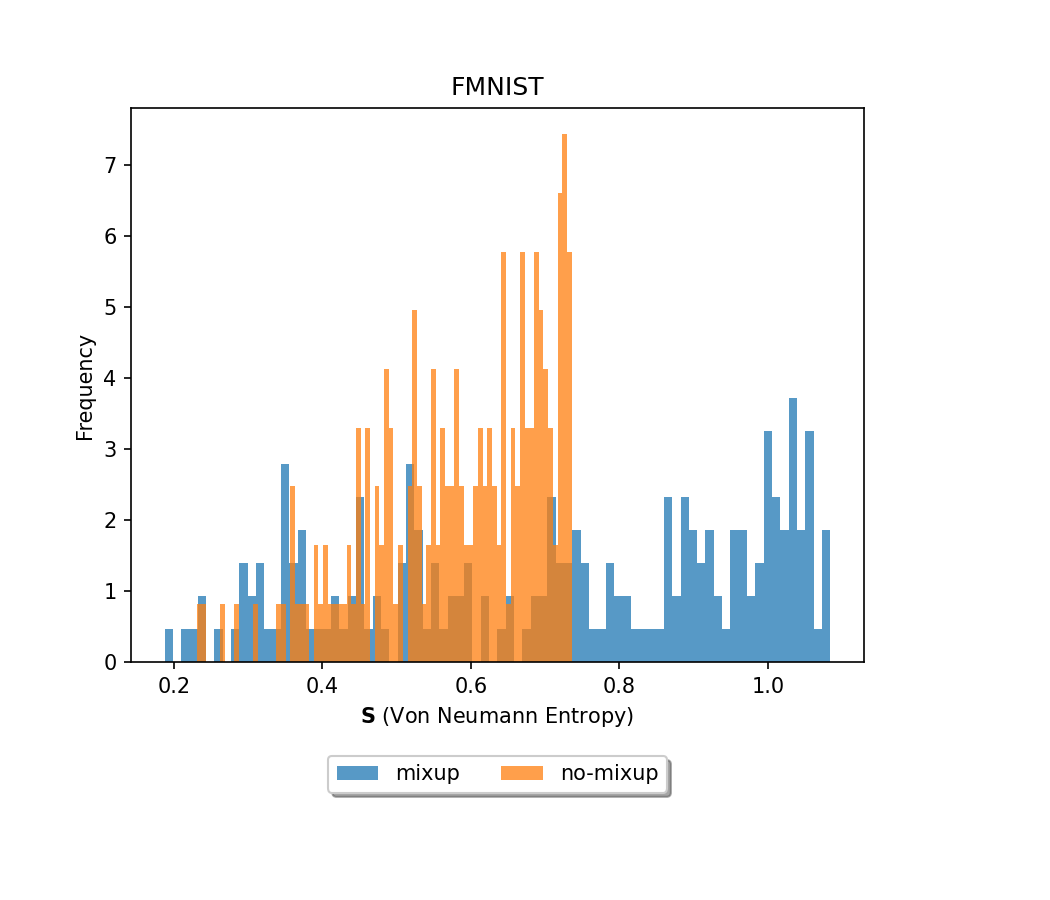}
\includegraphics[width=0.24\linewidth]{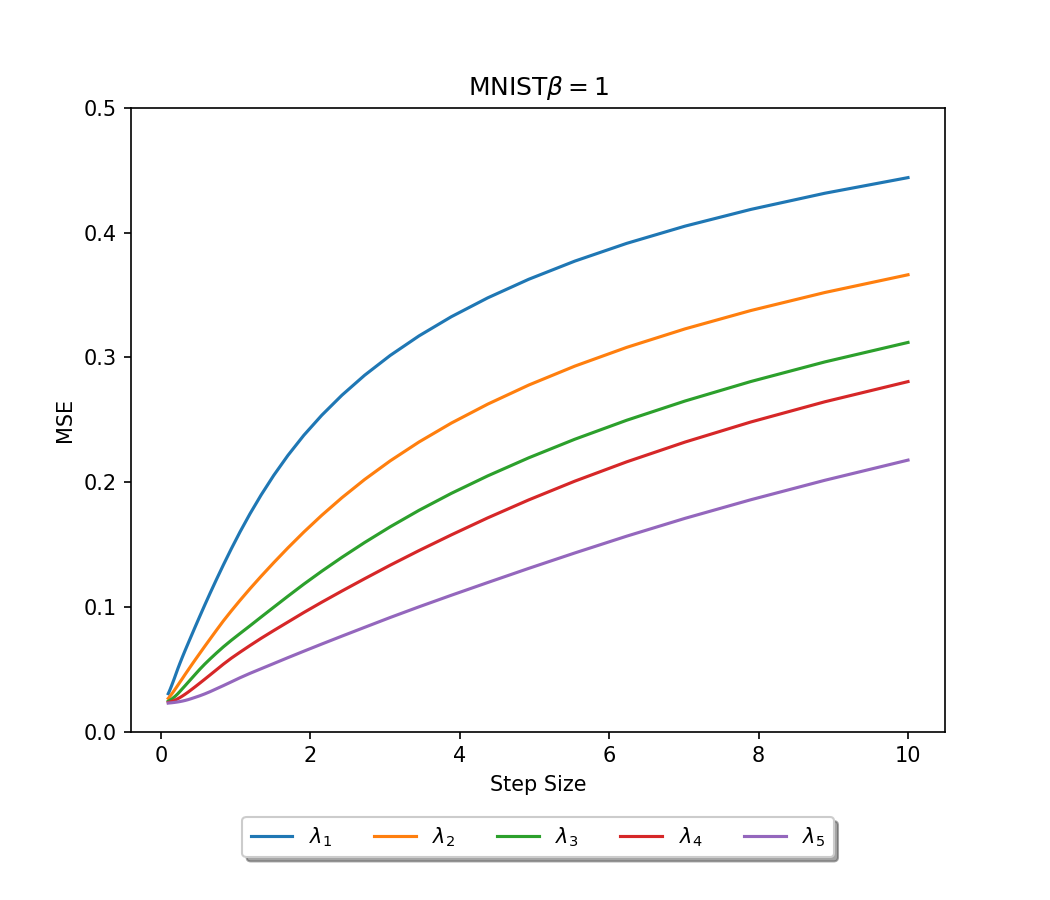}
\includegraphics[width=0.24\linewidth]{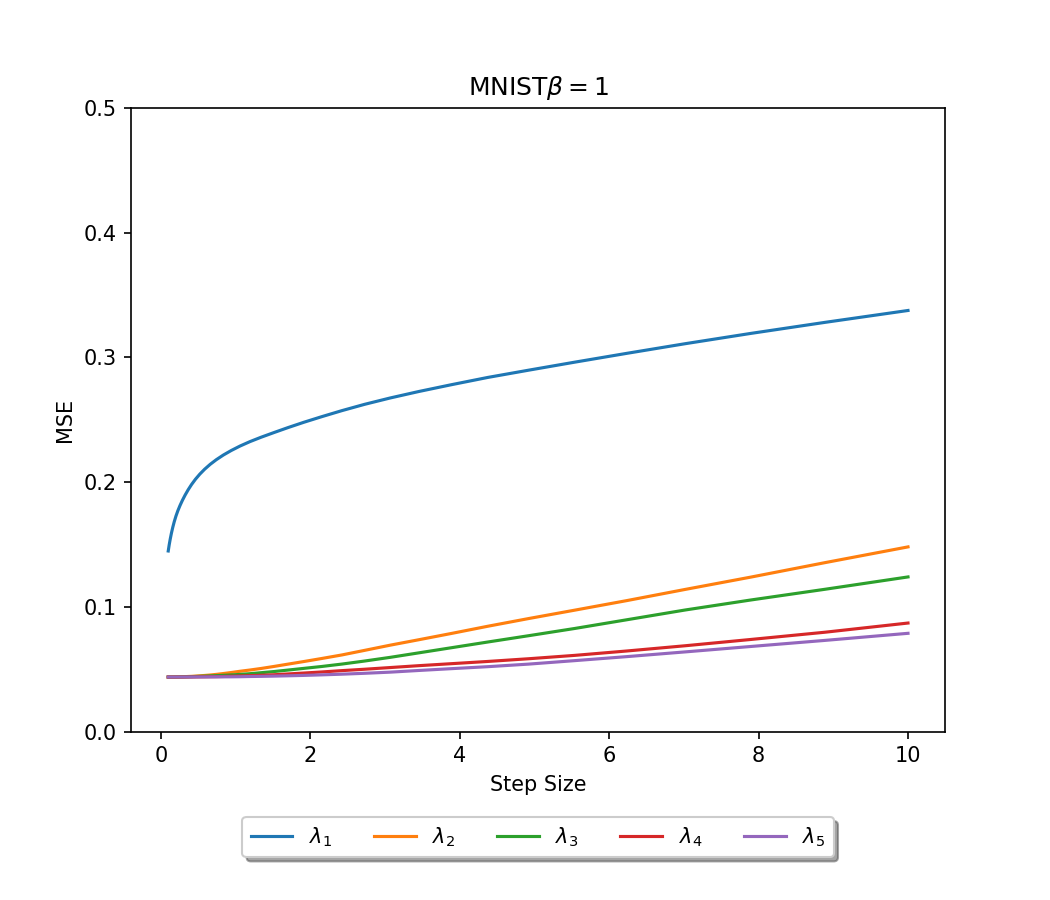}
\includegraphics[width=0.24\linewidth]{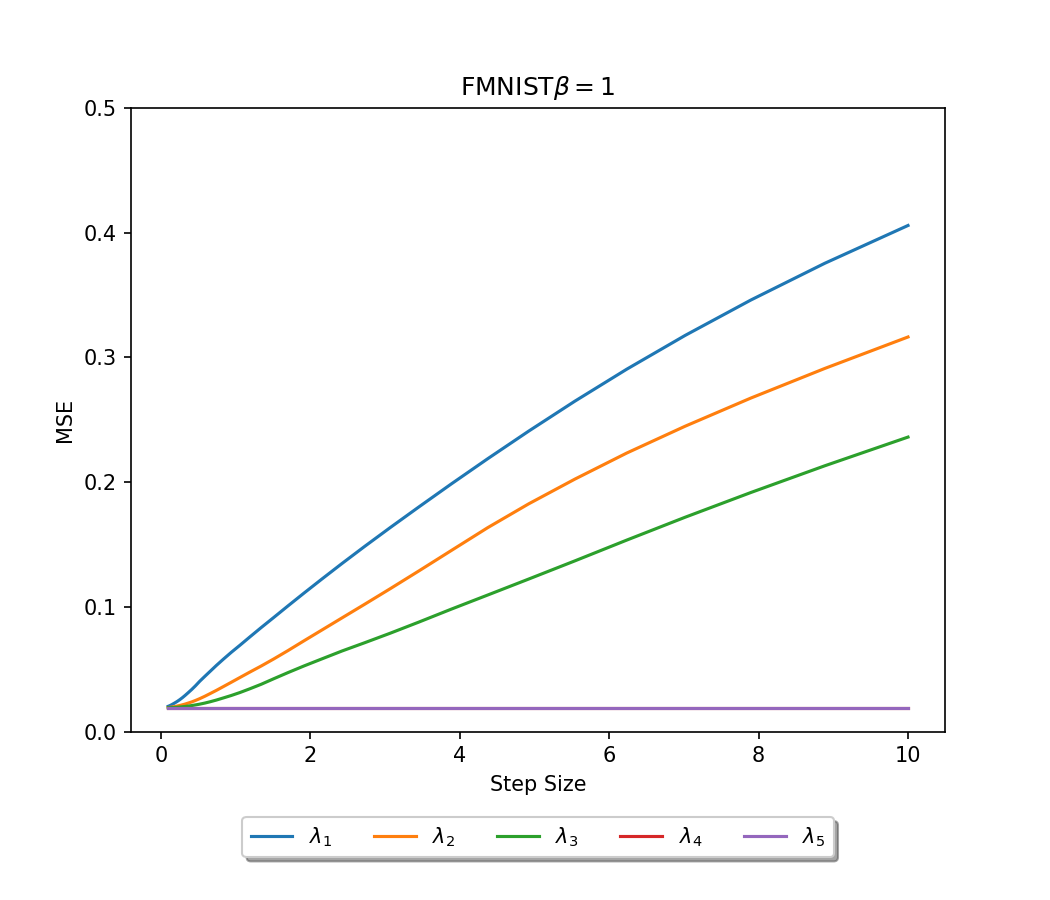}
\includegraphics[width=0.24\linewidth]{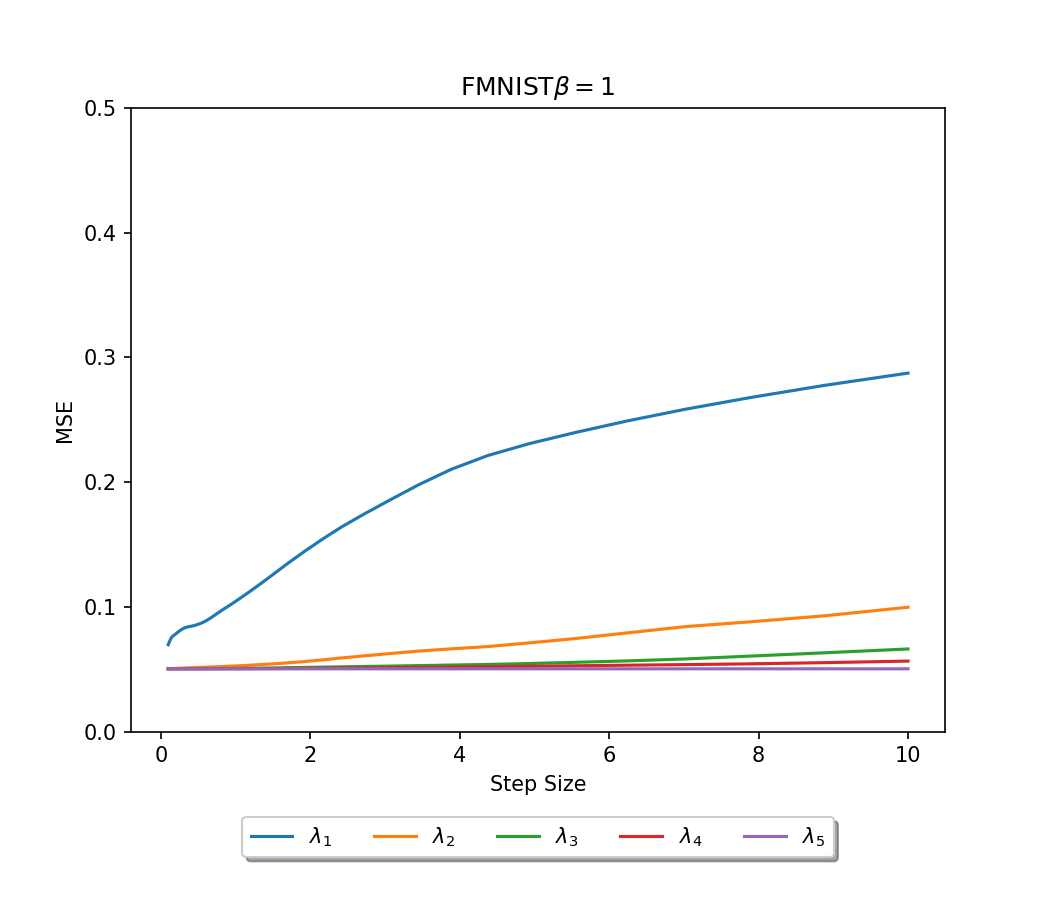}
    \caption{ Figure (a), first row, two columns are histograms of spectral radius and Von Neumann entropy (on test samples) with and without mixup regularisation. Second row, we corrupt the test images along the top five eigendirections (denoted by $\lambda_1, \lambda_2, \lambda_3, \lambda_4, \text{ and } \lambda_5$) with an increasing step size for different values of $\beta$. The plots describe the average MSE across test samples. We observe mixup suppresses the spectral radius, and the eigenspectrum distribution gets more isotropic. Figure (b) demonstrates similar observations on the FashionMNIST dataset.}
    \label{fig:evaluationmixup}
\end{figure*}

Increasing the $\beta$ improves the contribution of the uncertainty term, showcasing that the probabilistic encoder-decoder model is more robust than its deterministic counterpart. The stochastic metric tensor comprises two terms resulting from the latent space's variational distribution. For small values of $\beta$, the encoder does not do well in quantifying the uncertainty and fails to match the prior; as an outcome, the latent space is more distorted, resulting in empty and low-density regions given by dominating eigendirections of the metric tensor. The distortions are reduced for a high value of $\beta$; accordingly, the second term can account for the uncertainty in latent space preventing the eigenvalues from getting large.

In Supplementary Section~\ref{apsec:extraresults}, we report the results for the case where latent space is equipped with a metric tensor $\GBRV_{\zBRV}$ induced by the decoder network. We observe a higher value of spectral radius when incorporating the geometry of the decoder, which implies the latent space is locally curved. 

\subsubsection{Latent \textit{mixup}}
In $\beta-$VAE, a higher value of $\beta$ reduces the \textit{posterior-prior gap} but degrades the reconstruction quality. We propose filling up the low or zero-density region in the latent space as an alternative solution to the latent distortion problem. To this end, we utilise a \textit{mixup} training strategy. Mixup is a simple and powerful data augmentation technique that has been shown to improve the robustness and generalisation capabilities of classification models~\citep{zhang2017mixup, lamb2019interpolated, verma2019manifold}. For a given pair of distinct data points $\xBRV_i$ and $\xBRV_j$, we use linear interpolation between them to draw samples in the low-density region; likewise, we take the linear interpolation between the encodings $\zBRV_i$ and $\zBRV_j$, of the respective samples,
\begin{align}
    \zBRV_{m} = \alpha \zBRV_i + (1-\alpha) \zBRV_j, \quad 
    \xBRV_{m} = \alpha \xBRV_i + (1-\alpha) \xBRV_j
\end{align}
where $\alpha$ is sampled from a Beta distribution $\gB(a,b)$ with shape parameters $a$ and $b$ that we set to $0.5$. We now introduce the following regularisation penalty to a VAE objective,
\begin{align}
    C = ||\zBRV_m - f_\phi(\xBRV_m)||_2 + ||g_\theta(\zBRV_m) - \xBRV_m||_2
\end{align}

where the first term is to force the encoder to match mixing in the input space, and the second term is to force the decoder to match mixing in the latent space to mixing in the reconstruction space. Combining the two loss terms lets the encoder fill the empty region of latent space such that the linear interpolation in a latent space corresponds to linear interpolation in input space. Thus reducing the gap between the posterior and a prior and preventing the decoder from generating unconstrained output. We report the robustness scores for increasing step size $\delta$.

Figure~\ref{fig:evaluationmixup} shows the comparison of robustness scores with mixup training on MNIST and FashionMNIST datasets. We observe mixup suppresses the spectral radius for two datasets, thus improving the robustness. Figure~\ref{fig:evaluationmixup2} further demonstrates the qualitative performance of VAE trained with a \textit{mixup} loss.

\begin{figure*}[h]
\small{\textbf{FashionMNIST}: (a) Original \qquad \qquad (b) Reconstruction  \qquad \qquad \qquad \textbf{MNIST}: (c) Original \qquad \qquad (d) Reconstruction} \qquad
\begin{multicols}{4}
\centering
 \includegraphics[height=3cm, width=2.75cm, trim={0cm 2,1cm 0cm 0cm},clip]{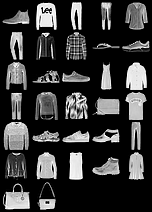}\par
 \includegraphics[height=3cm, width=2.75cm,trim={0cm 2,1cm 0cm 0cm},clip]{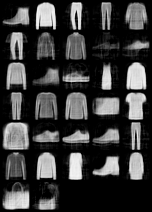}\par
\includegraphics[height=3cm, width=2.75cm,trim={0cm 2,1cm 0cm 0cm},clip]{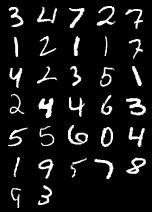}\par
\includegraphics[height=3cm, width=2.75cm,trim={0cm 2,1cm 0cm 0cm},clip]{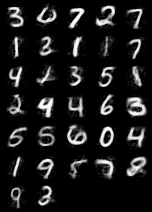}
\par 
\end{multicols}
    \caption{Here, we report a qualitative evaluation of training VAE with a mixup loss. We report the input samples and their respective reconstructions. On the left are the results of the FashionMNIST and the rights of MNIST.}
    \label{fig:evaluationmixup2}
\end{figure*}

\section{CONCLUSION, LIMITATIONS AND FUTURE SCOPE}
We have presented a geometrical perspective of adversarial attacks and introduced scores for measuring the robustness of VAEs. We have shown that the sensitivity of the encoder at a given input depends on eigendirections of stochastic pullback metric tensor that an adversary can exploit to design an attack. We proposed evaluation scores using the spectral radius and Von Neumann's entropy of a pullback metric tensor. Moreover, we demonstrate the scores correlate with parameter $\beta$ of $\beta-$VAE, providing geometrical insights into the robustness due to parameter $\beta$ of $\beta-$VAE. 

A caveat with $\beta-$VAE is that increasing the $\beta$ trades off the capacity of representations with the quality of reconstruction resulting in over-smooth reconstructions for a large value of $\beta$, which further implies a tradeoff between the spectral radius and robustness. To circumvent the above issue, we utilised a \textit{mixup} training scheme that fills the empty region of latent space and improves the robustness measured in terms of proposed scores. We want to remark \textit{mixup} does not guarantee that all empty or low-density regions are filled in the latent space. Here coverage depends on the regularisation hyperparameter as well as the size of the training dataset. Thus, it is still possible the regions of zero-low density exist. However, the overall training scheme reduces latent distortions. 

Recently few mechanisms have been proposed to reduce the distortion in latent space and improve the robustness of VAEs~\citep{willetts2019improving, kuzina2021diagnosing}. We hypothesise that the benefits of such robustness measures can be better established geometrically by investigating their pullback metric tensors. Another recent work~\citep{kuzinaalleviating} shows the adversarial attacks on VAEs target to move latent encoding to the region of low or zero density. Moreover, they proposed using a Markov Chain Monte Carlo (MCMC) as a correction scheme during the inference time. We wish to investigate the applicability of the proposed scores to the above methods in future work.

A limitation of our current work is that we only consider the unsupervised attack when a target sample is unknown. Moreover, there can be different forms of attack by replacing $l_2-$norm with more general $p-$norms. We wish to study these in future work.

\subsubsection*{Acknowledgements}
This work was partly supported by an unconditional gift from Huawei Noah's Ark Lab, London.

\bibliography{aistats2023}

\appendix
\onecolumn
\aistatstitle{Supplementary Materials}

\section{IMPLEMENTATION DETAILS}
\label{implement}
We use PyTorch~\cite{NEURIPS2019_9015} to implement our work, and all networks are trained on a single 11GB Nvidia RTX 2080 GPU.
For MNIST and FashionMNIST datasets,  we use the same architecture across all the experiments. The encoder network is a four-layer multi-layered perceptron (MLP) with $256$, $256$, $512$ and $32$ hidden units. The latent space distribution is a multivariate Gaussian with mean and standard deviation parameterised by two $32\times 32$ linear mappings. We use the standard zero mean and unit covariance prior on the latent space. The decoder network is the inverse of an encoder with $32$, $512$, $256$ and $256$ hidden units. We use \texttt{tanh} as an activation function and batch-normalisation~\cite{ioffe2015batch} before all activations. 

For the CelebA dataset, the encoder is a convolutional neural network (CNN) with four convolution layers with an increasing number of filters $32$, $64$, $128$, $256$, and $512$ followed by a dense layer that maps to a latent space. The latent distribution is a multivariate Gaussian with mean and standard deviation given by $128\times 128$ linear mappings. The decoder is an inverse of encoder architecture with a dense layer that maps to $1024$ hidden units followed by transpose convolution layers with decreasing number of filters $512$, $256$, $128$, $64$ and $32$. All convolution layers are followed with a \texttt{tanh} activation function and batch-normalisation layer. For all models, we use Adam optimiser~\cite{kingma2014adam} with a learning rate of $0.003$.

\section{EXTENDED RESULTS}
\label{apsec:extraresults}
We also conducted experiments for the case when latent space is locally curved that is $\GBRV_{\zBRV}=\JBRV_{g_{\phi}(\zBRV)}^T \JBRV_{g_{\phi}(\zBRV)}$ where $g_{\phi}$ is a decoder network with parameters $\phi$. For a stochastic decoder mapping $g_{\phi} = \muB(\zBRV) + \epsilonB \odot \sigmaBRV_{\phi} (\zBRV)$ the metric tensor $\GBRV_{\zBRV}$ can further be expressed as $\GBRV_{\zBRV}=\JBRV_{\muB_{\phi}(\zBRV)}^T\JBRV_{\muB_{\phi}(\zBRV)}  + \JBRV_{\sigmaB_{\phi}(\zBRV)}^T\JBRV_{\sigmaB_{\phi}(\zBRV)}$. We use this expression of pullback metric tensor in the latent space to measure infinitesimal distance which results in the following expression of combined pullback metric tensor,
\begin{equation}
\label{eq:pullnew}
     \hat{\GBRV}_{\xBRV} =  \JBRV_{\muB_{\theta}(\xBRV)}^T \GBRV_{\zBRV} \JBRV_{\muB_{\theta}(\xBRV)}  + \JBRV_{\sigmaB_{\theta}(\xBRV)}^T \GBRV_{\zBRV} \JBRV_{\sigmaB_{\theta}(\xBRV)}
\end{equation}

Figure~\ref{fig:attackdec}-\ref{fig:attackdec2} show the results on MNIST and FashionMNIST. Figure~\ref{fig:evaluationdec} further discusses the analysis of robustness scores.

\begin{figure*}[h]
\small{\textbf{MNIST}: (a) Original  $\downarrow$ Reconstruction  \qquad(b) Corrupted $\delta_1$ $\downarrow$ Reconstruction \qquad(c) Corrupted $\delta_2$ $\downarrow$ Reconstruction}
\begin{multicols}{3}
 \includegraphics[height=4cm, width=5.25cm]{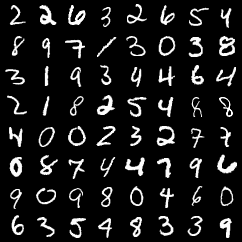}
 \includegraphics[height=4cm, width=5.25cm]{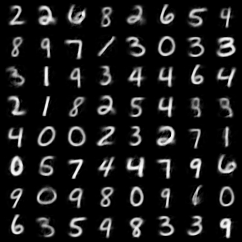}\par
\includegraphics[height=4cm, width=5.25cm]{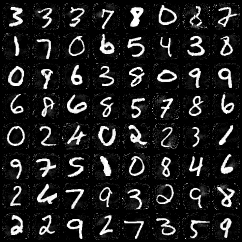}  
\includegraphics[height=4cm, width=5.25cm]{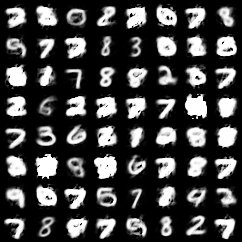}
\par
\includegraphics[height=4cm, width=5.25cm]{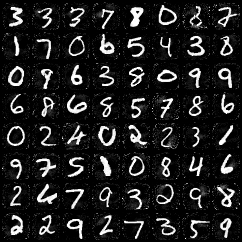}
\includegraphics[height=4cm, width=5.25cm]{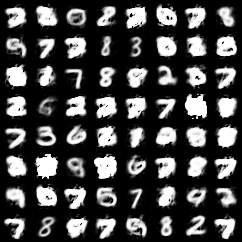} 
\end{multicols}
    \caption{Illustration of adversarial attack along the dominant eigenvector of a stochastic pullback metric tensor given by Equation~\ref{eq:pullnew}. We evaluate the reconstruction for original images and its two corrupted versions with different step sizes $\delta_1=0.5233$ and $\delta_2=0.7443$.}
    \label{fig:attackdec}
\end{figure*}

\begin{figure*}[h]
\small{\textbf{FMNIST}: (a) Original  $\downarrow$ Reconstruction  \qquad(b) Corrupted $\delta_1$ $\downarrow$ Reconstruction \qquad(c) Corrupted $\delta_2$ $\downarrow$ Reconstruction}
\begin{multicols}{3}
 \includegraphics[height=4cm, width=5.25cm]{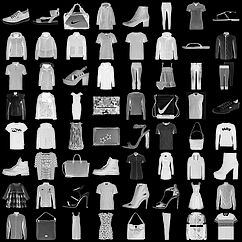}
 \includegraphics[height=4cm, width=5.25cm]{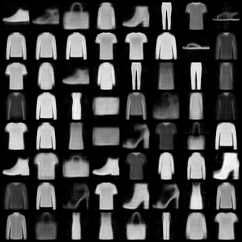}\par
\includegraphics[height=4cm, width=5.25cm]{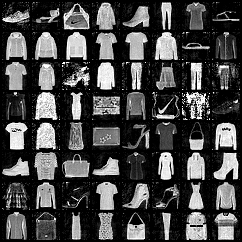}  
\includegraphics[height=4cm, width=5.25cm]{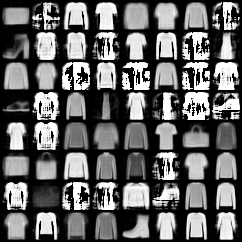}
\par
\includegraphics[height=4cm, width=5.25cm]{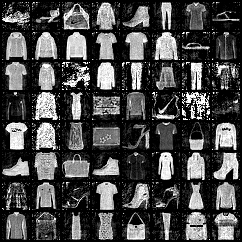}
\includegraphics[height=4cm, width=5.25cm]{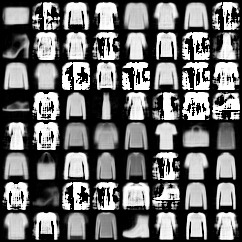} 
\end{multicols}
    \caption{Illustration of adversarial attack along the dominant eigenvector of a stochastic pullback metric tensor given by Equation~\ref{eq:pullnew}. We evaluate the reconstruction for original images and its two corrupted versions with different step sizes $\delta_1=0.5233$ and $\delta_2=0.7443$.}
    \label{fig:attackdec2}
\end{figure*}

\begin{figure*}[ht!]
    \centering
    \small{(a) Robustness evaluation of $\beta-$VAE on MNIST.}\\
    \includegraphics[width=0.24\linewidth]{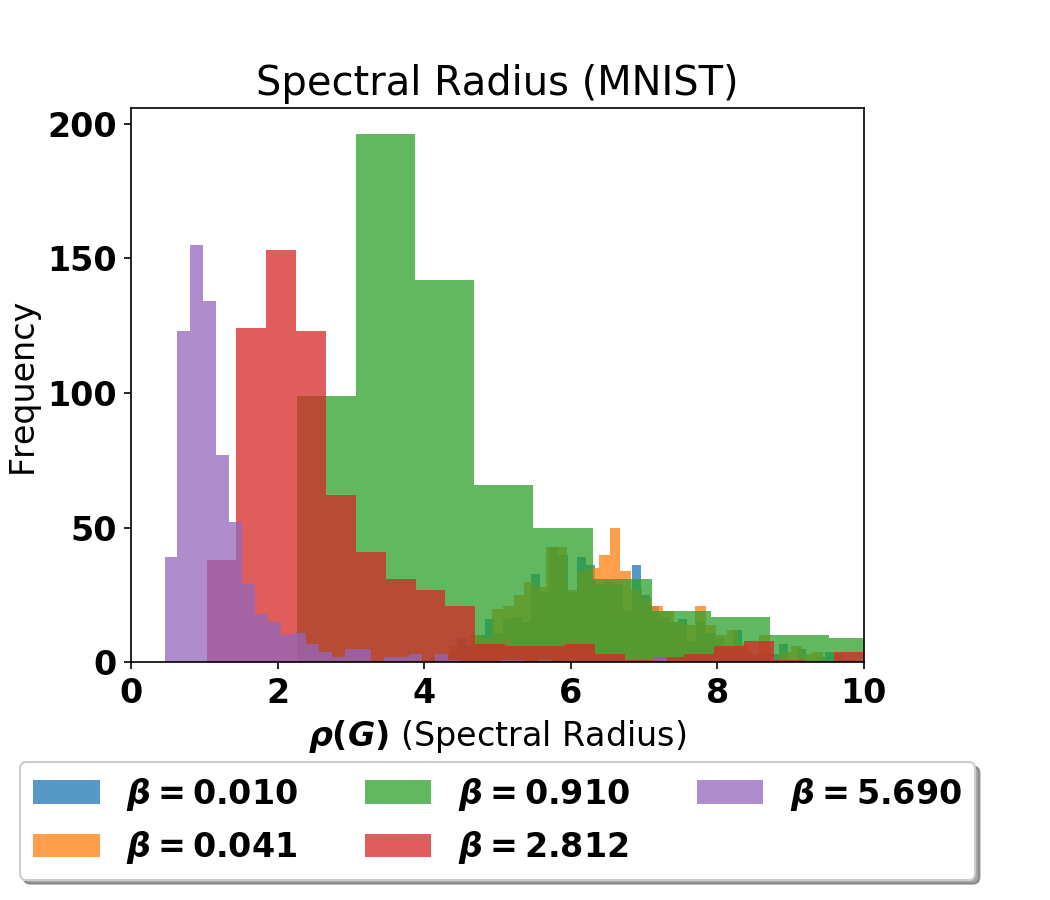}
    \includegraphics[width=0.24\linewidth]{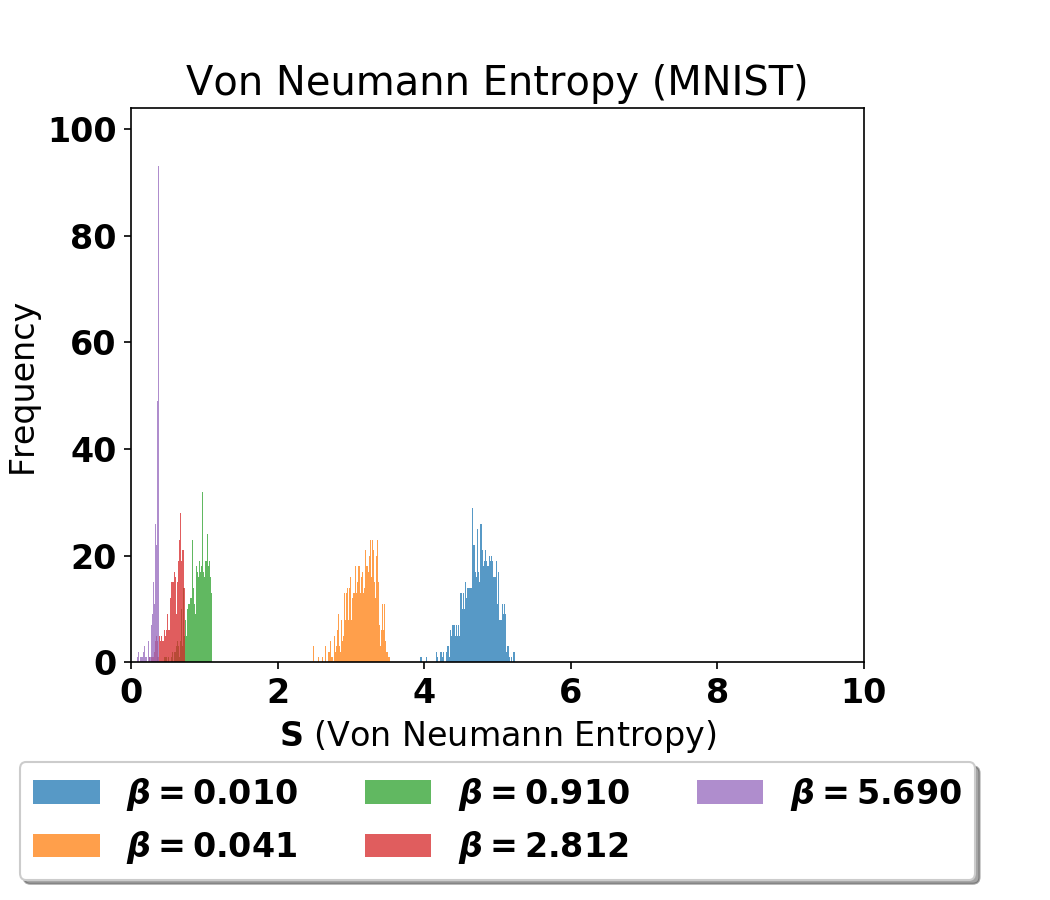}
    \includegraphics[width=0.24\linewidth]{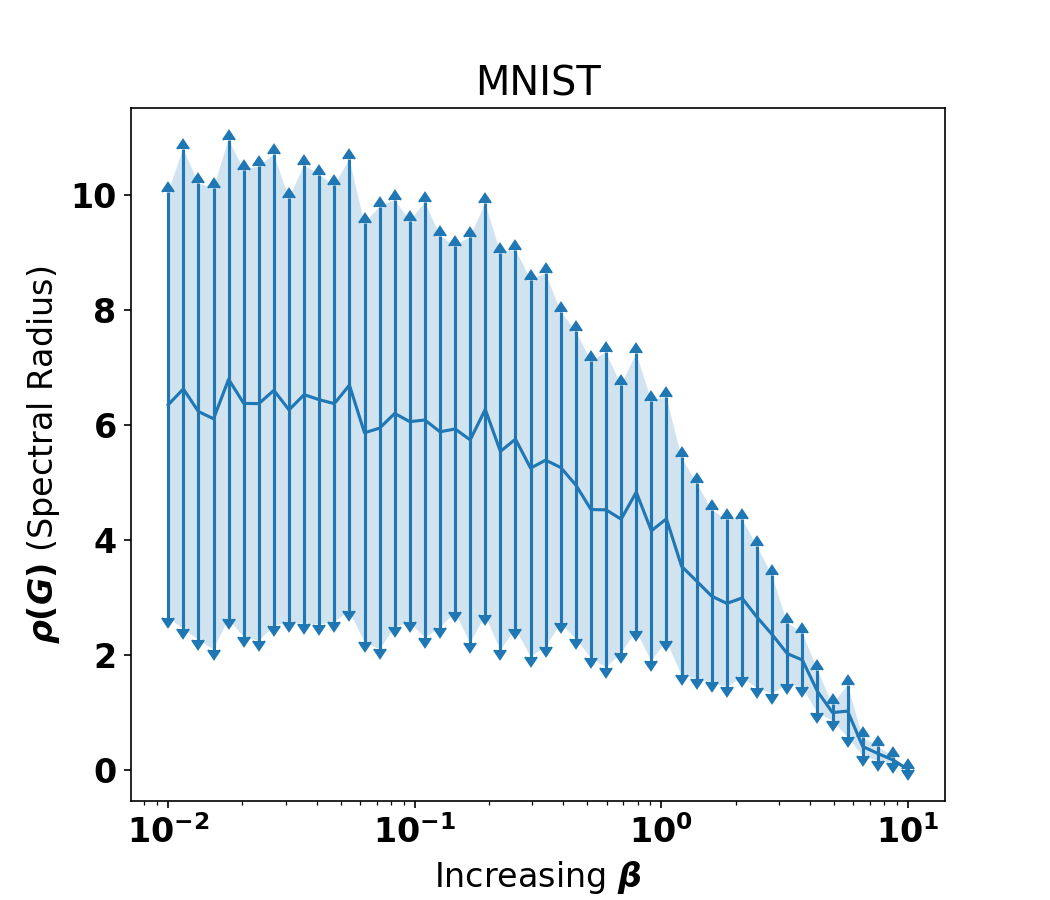}
    \includegraphics[width=0.24\linewidth]{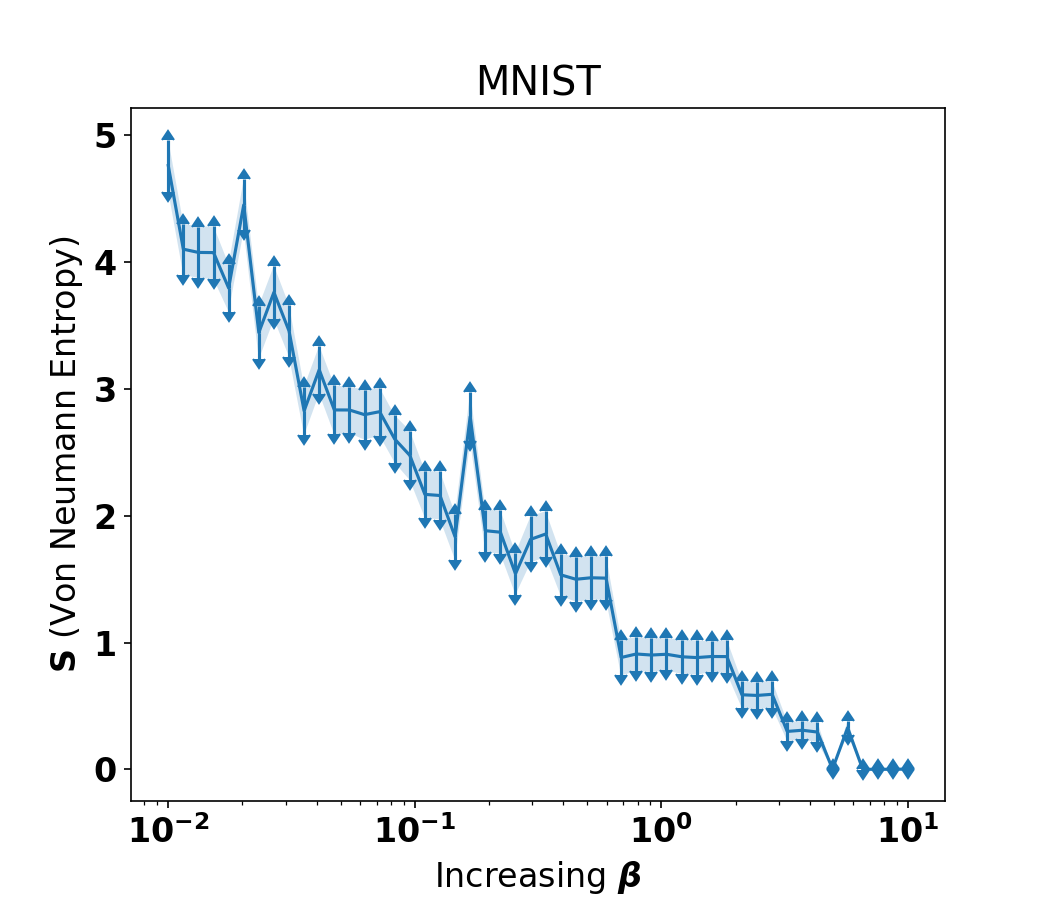}
    \includegraphics[width=0.24\linewidth]{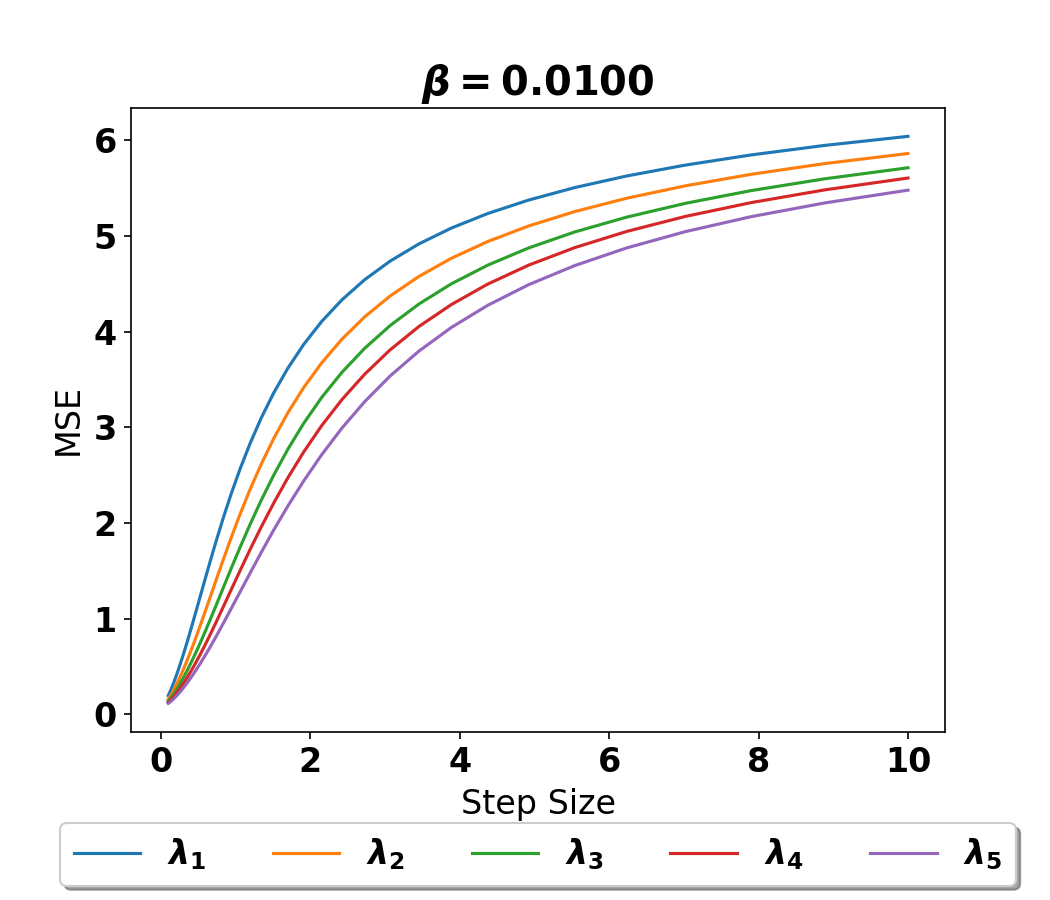}
    \includegraphics[width=0.24\linewidth]{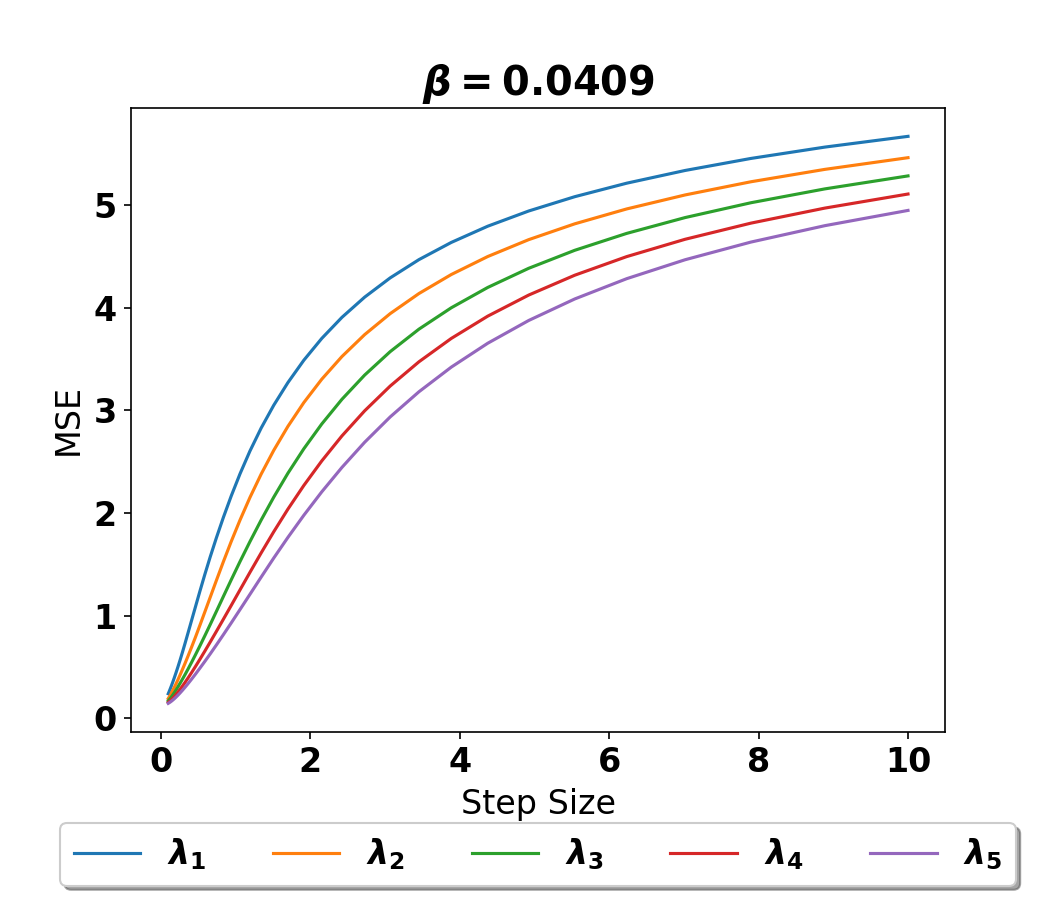}
    \includegraphics[width=0.24\linewidth]{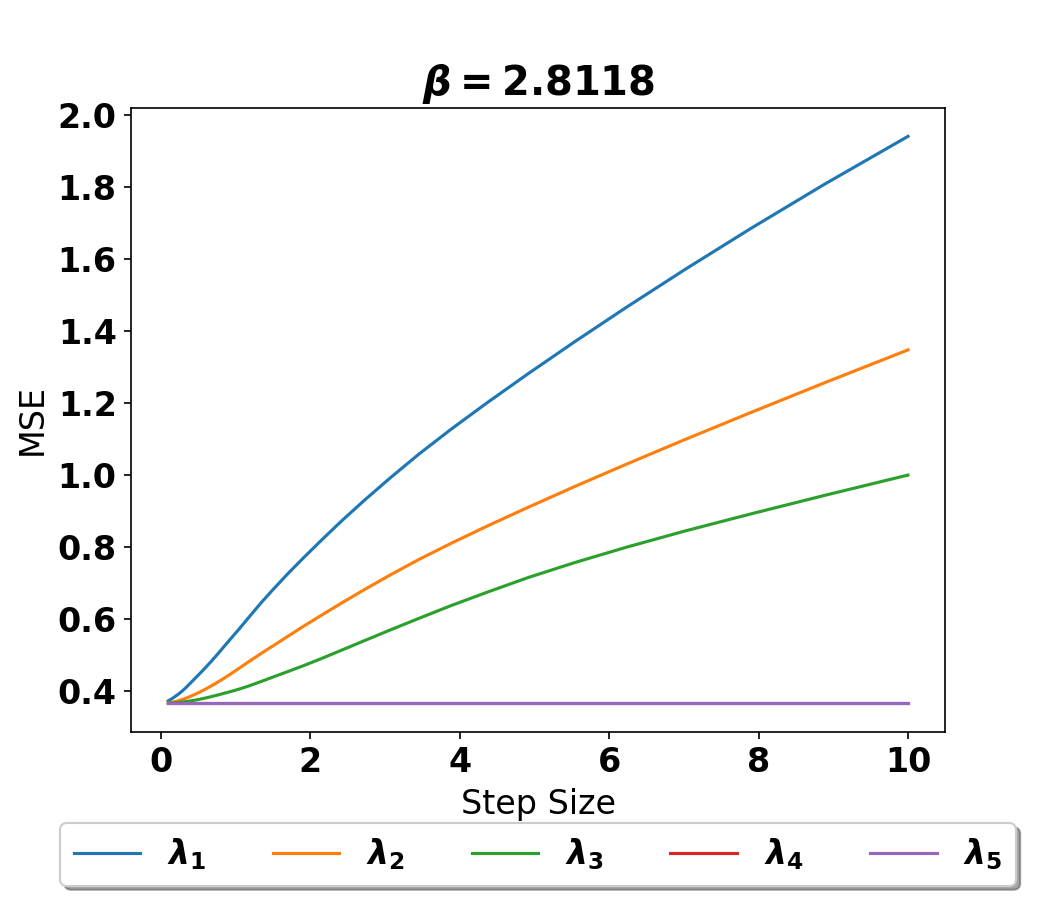}
    \includegraphics[width=0.24\linewidth]{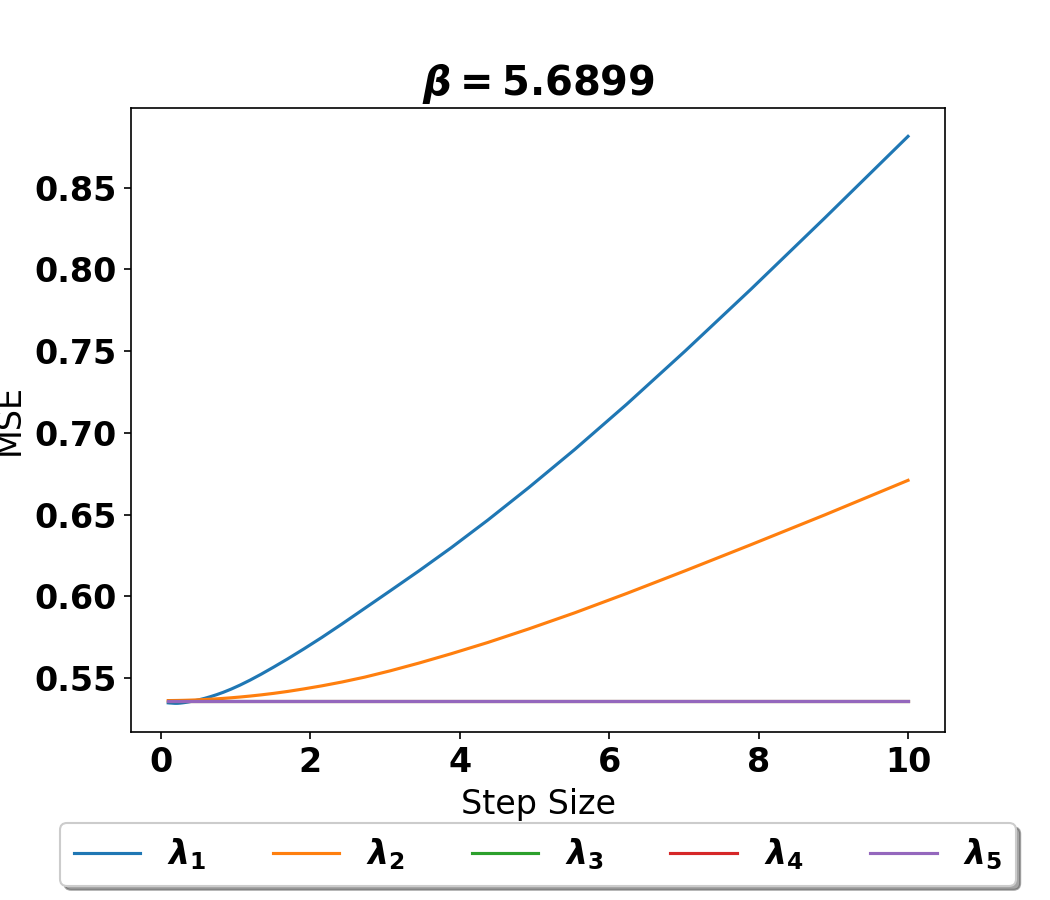}\\
    \small{(b) Robustness evaluation of $\beta-$VAE on FashionMNIST.}\\
    \includegraphics[width=0.24\linewidth]{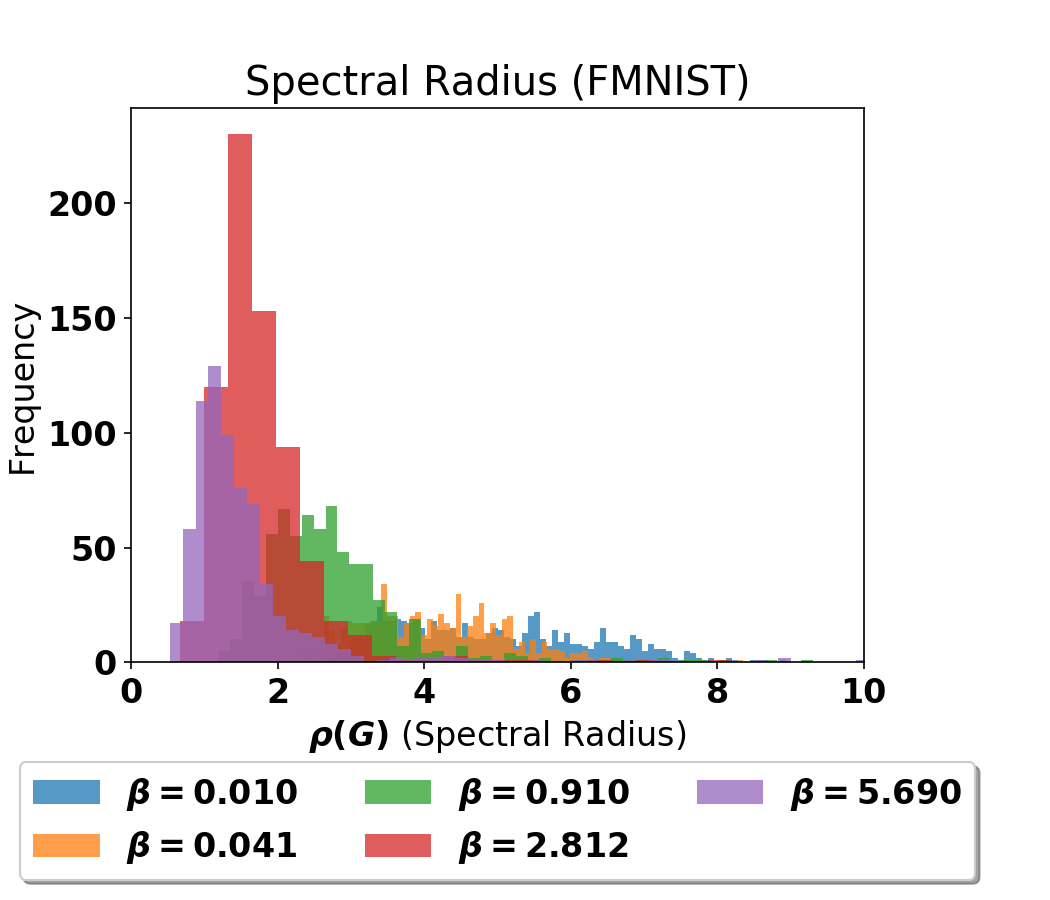}
    \includegraphics[width=0.24\linewidth]{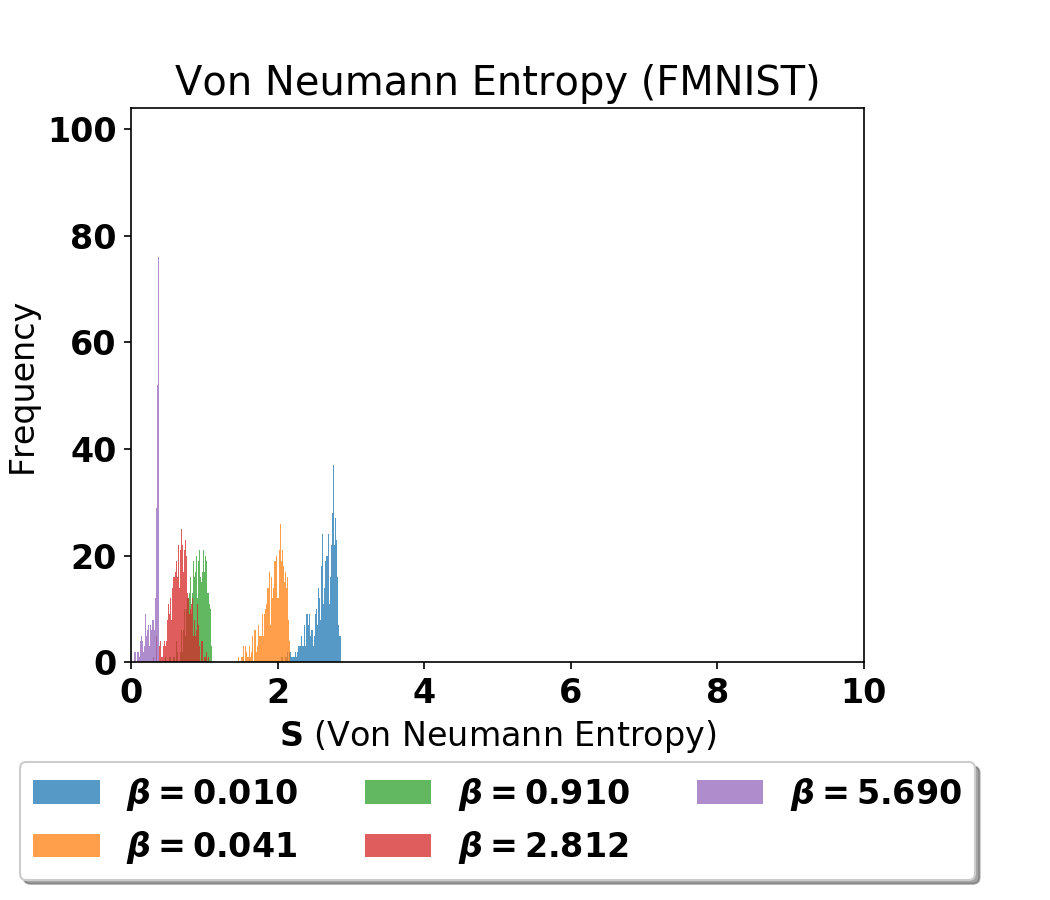}
    \includegraphics[width=0.24\linewidth]{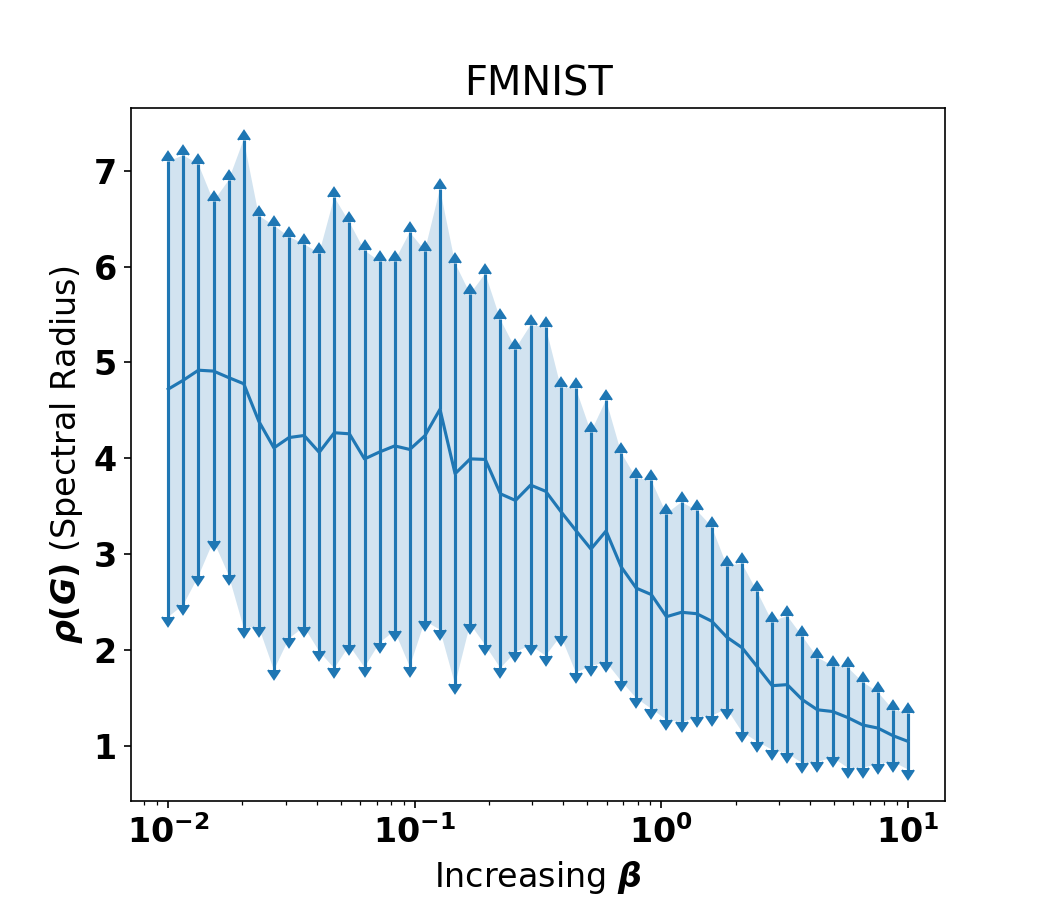}
    \includegraphics[width=0.24\linewidth]{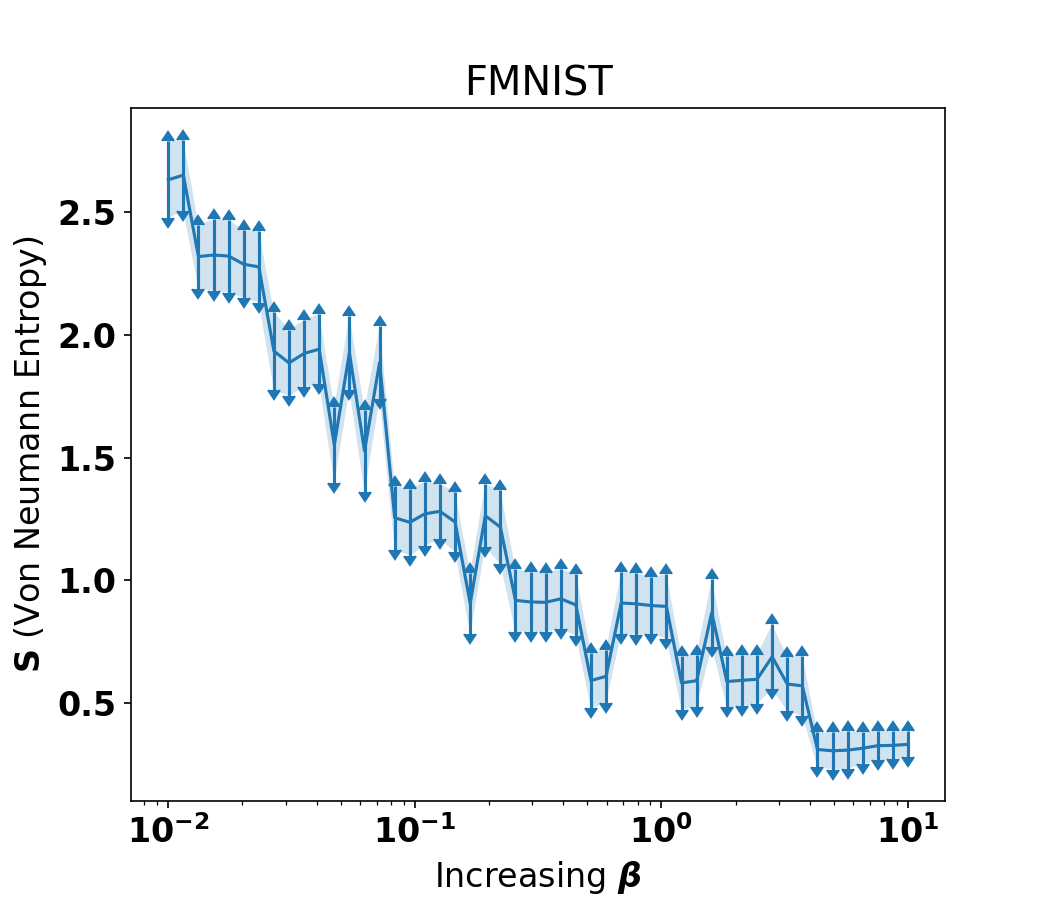}
    \includegraphics[width=0.24\linewidth]{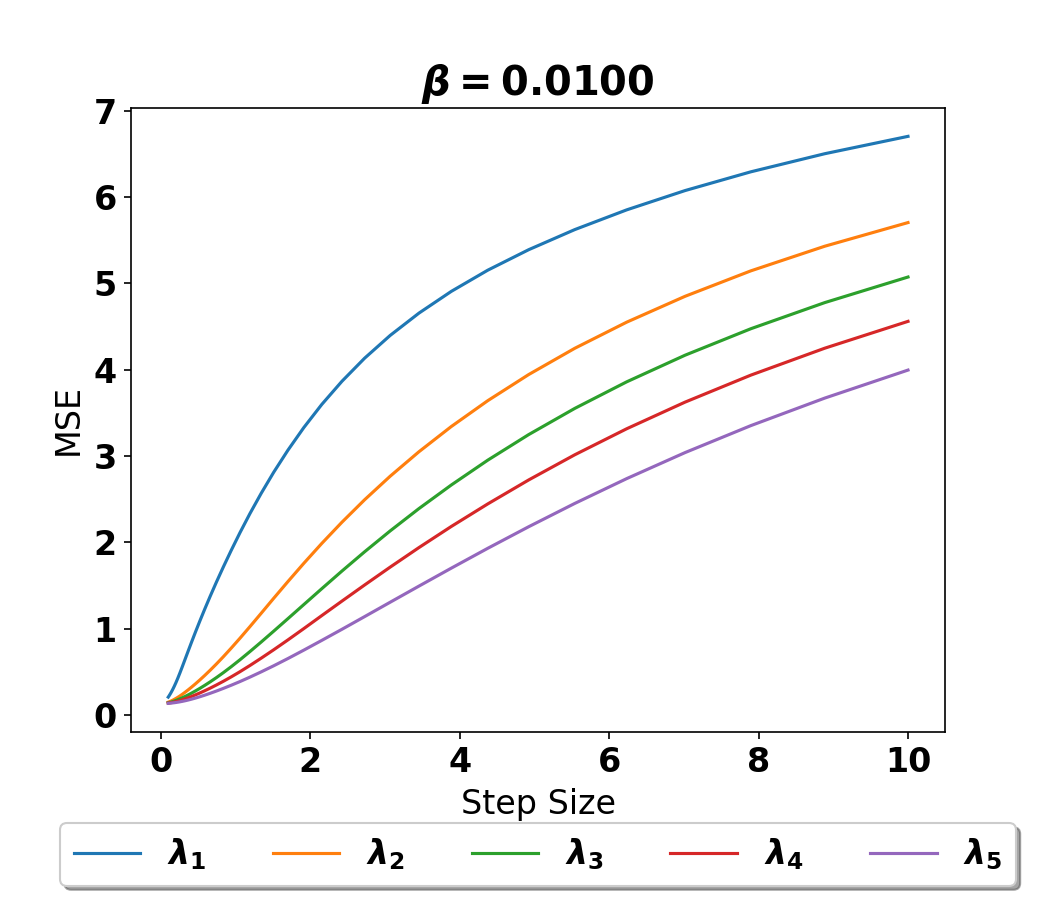}
    \includegraphics[width=0.24\linewidth]{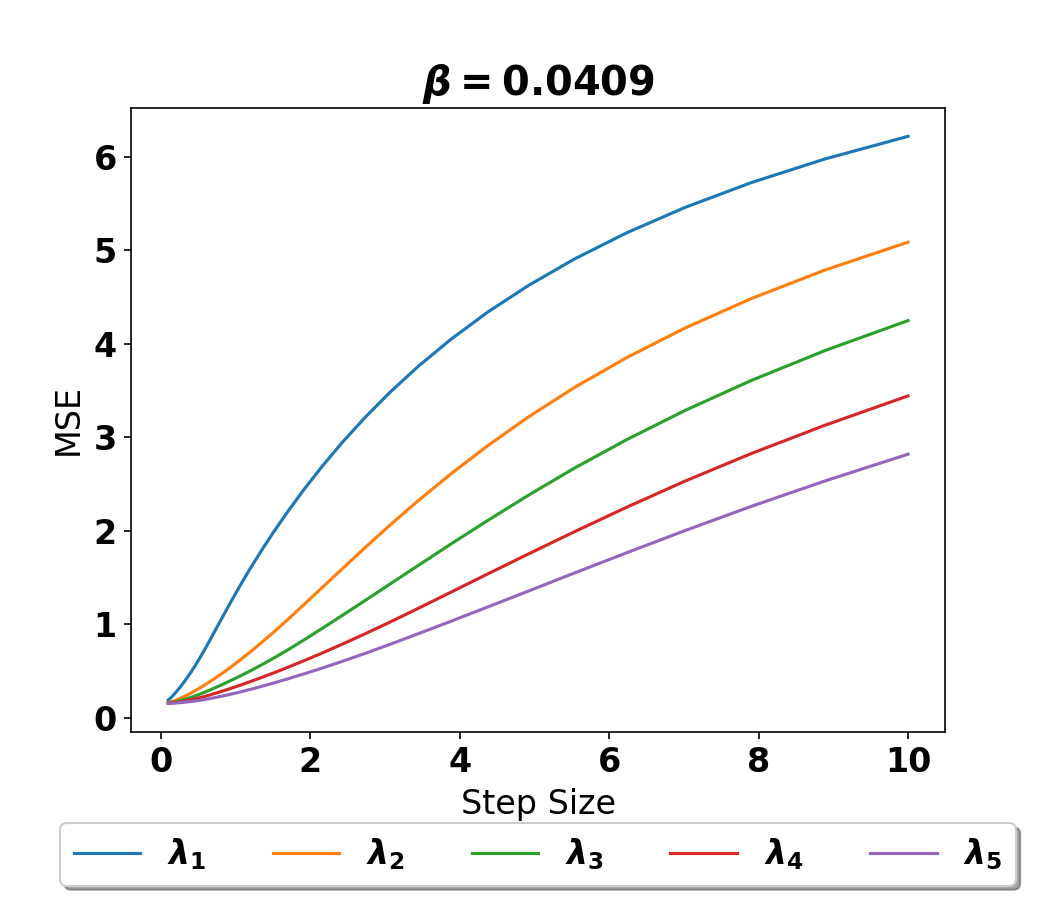}
    \includegraphics[width=0.24\linewidth]{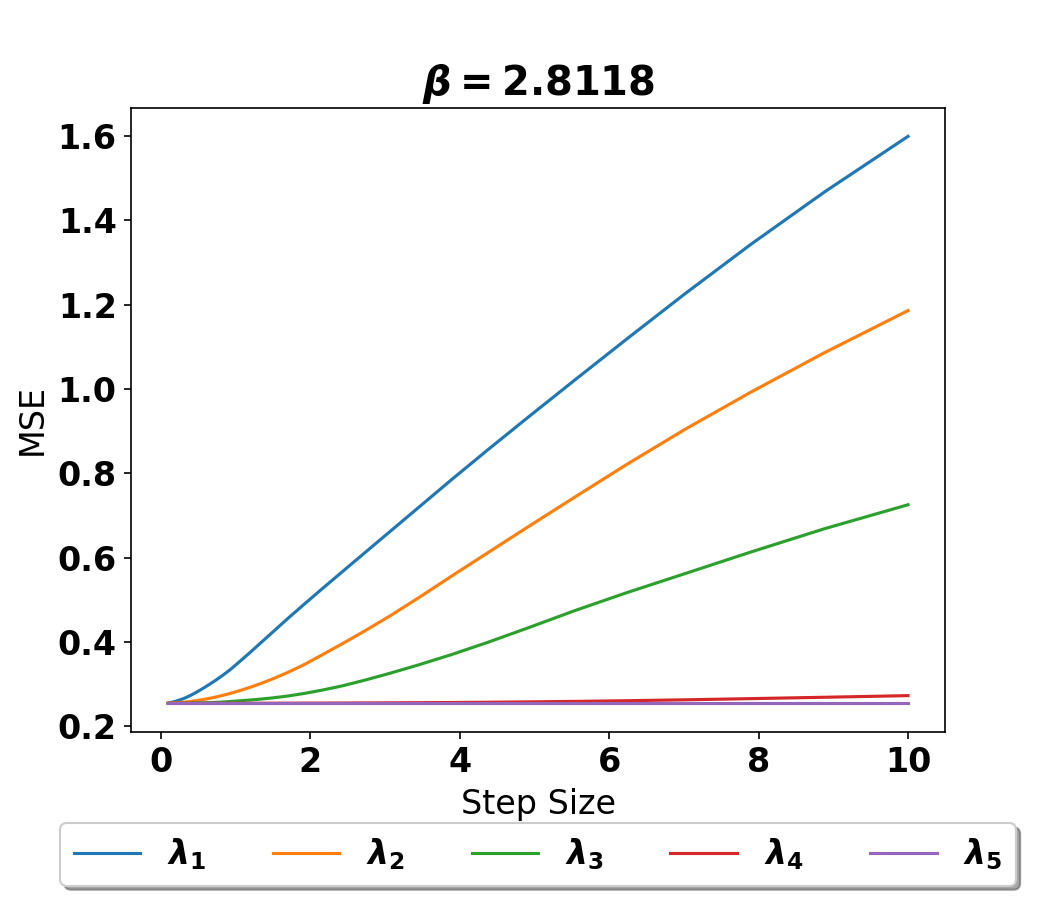}
    \includegraphics[width=0.24\linewidth]{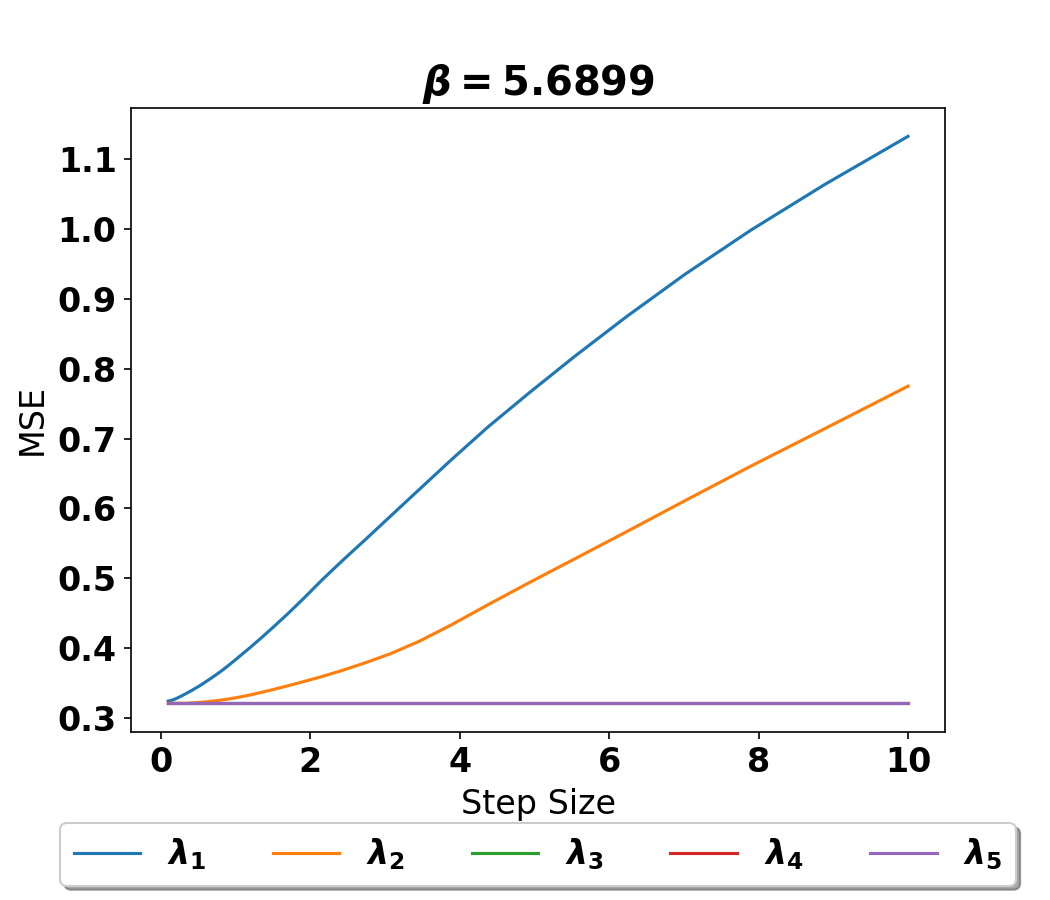}
    \caption{ Evaluation using stochastic pullback metric tensor given by Equation~\ref{eq:pullnew}. Figure (a), on the left, we report the histogram of spectral radius and Von Neumann entropy (on test samples) for different values of $\beta$ in $\beta$-VAE. On the right, we report the average of two scores across test samples for an increasing value of $\beta$. Increasing the value of $\beta$ suppresses the metric tensor's maximum eigenvalue, and the eigenspectrum distribution gets more isotropic. In the second row, we corrupt the test images along the top five eigendirections (denoted by $\lambda_1, \lambda_2, \lambda_3, \lambda_4, \text{ and } \lambda_5$) with an increasing step size for different values of $\beta$. The plots describe the average MSE across test samples. We observe for a higher value of $\beta$, the average step size increases. Increasing the value of $\beta$ reduces the \textit{posterior-prior gap}, minimising distortion in the latent space. Figure (b) demonstrates similar observations on the FashionMNIST dataset.}
    \label{fig:evaluationdec}
\end{figure*}

\end{document}